\pdfoutput=1
\documentclass[twoside]{article}

\usepackage[preprint]{aistats2027}

\usepackage[numbers,square,sort&compress]{natbib}

\usepackage[hyphens]{url}
\usepackage{graphicx}
\usepackage{caption}
\usepackage{subcaption}  %

\usepackage{booktabs}
\usepackage{multirow}
\usepackage{xcolor}
\definecolor{rankbest}{rgb}{0.0,0.55,0.0}
\definecolor{ranksecond}{rgb}{0.0,0.0,0.85}
\definecolor{rankworst}{rgb}{0.80,0.0,0.0}
\newcommand{\best}[1]{\textbf{\textcolor{rankbest}{#1}}}
\newcommand{\second}[1]{\textcolor{ranksecond}{#1}}
\newcommand{\worst}[1]{\textcolor{rankworst}{#1}}
\newcommand{\gap}[2]{#1\textsubscript{$\pm$#2}}

\usepackage{tikz}
\usetikzlibrary{positioning,arrows.meta,calc,backgrounds,fit}

\usepackage{amsmath}
\usepackage{amssymb}
\usepackage{amsthm}
\usepackage{mathtools}
\usepackage{cleveref}

\theoremstyle{plain}
\newtheorem{proposition}{Proposition}

\theoremstyle{definition}
\newtheorem{definition}{Definition}

\theoremstyle{remark}

\newcommand{\R}{\mathbb{R}}
\newcommand{\E}{\mathbb{E}}
\newcommand{\Cov}{\operatorname{Cov}}
\newcommand{\phat}[1]{\hat{p}_{#1}}

\newcommand{\Y}{\mathcal{Y}}
\newcommand{\ctx}{\mathcal{D}}

\begin{document}

\runningtitle{Do Tabular Foundation Models Agree with Themselves?}
\runningauthor{Kl\"otergens, Yalavarthi, Schmidt-Thieme, Hanika}

\twocolumn[
\aistatstitle{Do Tabular Foundation Models Agree with Themselves?}

\aistatsauthor{Christian Kl\"otergens, Vijaya Krishna Yalavarthi,
  Lars Schmidt-Thieme, Tom Hanika}
\aistatsaddress{Institute of Computer Science, University of
  Hildesheim, Hildesheim, Germany}

]

\begin{abstract}
Tabular Foundation Models (TFMs) are the state of the art for many tabular prediction tasks. They are transformers that approximate the Bayesian posterior predictive distribution under a pretraining prior. TFMs become multivariate predictors autoregressively, by sampling one target and appending it to the features. This paper is the first to evaluate whether this setup induces a valid stochastic process. Informally, we ask: \emph{Do we get the same answer to the same question asked in two different ways?} The standard Kolmogorov consistency conditions cannot be checked on TFMs directly, since TFMs are unable to model certain conditional distributions. We therefore propose two adapted, testable requirements. Marginalization consistency demands that marginalized conditionals equal directly predicted marginals, and factorization consistency that both factorization orders yield the same joint. Since the data-generating distribution satisfies both, a perfect TFM would as well, and every violation lower-bounds how far the model's predictions are from the ground truth. Every TFM we evaluate violates both requirements, for classification and regression, on all datasets.
\end{abstract}

\section{Introduction}

Tabular Foundation Models (TFMs) rest on a claim about Bayesian inference. \citet{Muller2022.Transformers} showed that a transformer neural process~\citep{Nguyen2022.Transformer} trained to predict held-out targets on datasets drawn from a prior approximates the Bayesian posterior predictive distribution under that prior. The PFN blueprint now underwrites TFMs such as TabPFN~\citep{Hollmann2025.Accurate}, TabICL~\citep{Qu2025.TabICLa}, TabDPT~\citep{Ma2025.TabDPT} and most recently TabFM~\citep{Kong.TabFM}, which are the state of the art for many tabular prediction tasks. 
The prediction heads of current TFMs are restricted to univariate distributions. However, a value of one target can simply be appended as a feature to condition the prediction of another. Chaining such conditionals turns a univariate predictor into an autoregressive construction of a multivariate joint~\citep{Vetter2025.Effortless}. 
\begin{figure}[t]
    \centering
    \includegraphics[width=\columnwidth]{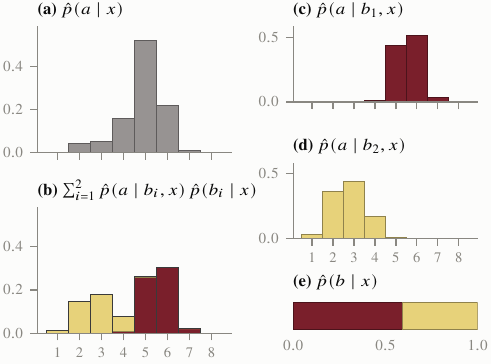}
    \caption{A marginalization-consistency violation on a single \textsc{Wine} test instance from TabPFNv3. The violation is observable when comparing panels \textbf{(a)} and \textbf{(b)}, which show distributions that should be equal, but are not. \textbf{(a)} represents the distribution of the \textsc{Quality-Rating} using all predictors but \textsc{Color}. The distribution depicted in \textbf{(b)} is computed by mixing the \textsc{Quality-Rating} prediction under the assumption that the \textsc{Color} is red \textbf{(c)} and white \textbf{(d)}, weighted by the probability for each color \textbf{(e)}.}\label{fig:wine-teaser}
\end{figure}

For conditional neural processes~\citep{Garnelo2018.Conditional}, such autoregressive constructions~\citep{Bruinsma2023.Autoregressive} are known not to represent a valid stochastic process. A family of finite-dimensional distributions defines a stochastic process if and only if it satisfies the Kolmogorov consistency conditions~\citep{Oksendal2003.Stochastic}: permuting the variables must not change their joint distribution, and marginalizing a variable out must recover the distribution of the remaining ones. A neural process whose predictions factorize over test points satisfies both conditions, if the context is fixed. However, once a neural process is chained autoregressively, permuting the targets can change the distribution~\citep{Young2026.Conditioning}. For TFMs the Kolmogorov conditions cannot be checked directly. The conditions marginalize a target while holding the context fixed, whereas a TFM cannot be queried with a column that is present in the context but absent at the test point. Marking the value as missing is not the same as leaving the column out. We therefore propose adjusted requirements in the spirit of the Kolmogorov conditions. Ultimately, all they ask is that predictions of a TFM do not contradict predictions of the same TFM\@.

Figure~\ref{fig:wine-teaser} shows what an inconsistency looks like using a single row from the \textsc{Wine} dataset as an example. We compare two predicted distributions of the \textsc{Quality-Rating} for some bottle of wine. The first distribution is predicted directly using all available features but ignoring the \textsc{Color}. The second distribution is computed by mixing the conditionals under the assumptions that the color is either red or white (which are all possible wine colors in the dataset). The mixture is weighted by the probability for the respective color. The true joint distribution over all columns in the dataset is unknown, yet any single joint distribution behind both queries would give them the same marginal distribution of \textsc{Quality-Rating}. TabPFNv3, however, computes two very different distributions, hence it disagrees with itself.

To guide the ongoing research on TFMs towards minimizing or even eliminating such inconsistencies, we formulate two requirements that any model whose predictions are induced by a well-defined joint must satisfy.
First, \emph{marginalization consistency}, the counterpart of the marginalization condition, which ties the model's conditionals back to its directly predicted marginals.
Second, \emph{factorization consistency}, the counterpart of the permutation condition, which demands that the two chain-rule orders build the same joint distribution. 
Unlike the classical conditions, both compare only predictions that a TFM can make. We show that both conditions follow from the law of total probability and the chain rule. 
Furthermore, we formally show that these two consistency requirements are closely connected. Factorization consistency implies marginalization consistency, but not vice versa. Hence, every observed marginalization failure already certifies that changing the order of factorization will lead to different distributions. Since the data-generating distribution satisfies both requirements, every violation also lower-bounds the distance between the model's predictions and the ground truth.

We also design experiments to evaluate whether and to what extent different TFM classifiers and (probabilistic) regressors violate each consistency requirement.
Every model we evaluate violates both conditions, on every dataset, for both classification and regression. Our results show that the recently released TabFM is the \emph{most consistent} classifier. However, this model currently does not offer probabilistic regression. For regressors we do not find a model that is the most consistent across datasets and requirements. Furthermore, our experiments show that our proposed measures of violation for each requirement strongly correlate with each other.
Our contributions are:
\begin{itemize}
    \item We adapt the consistency conditions of stochastic processes to the way TFMs are constructed, formulating two requirements that a TFM must satisfy for its predictions to be consistent.
    \item From these requirements we derive metrics that quantify the extent of violations without knowledge of the data-generating process, and show that they lower-bound the distance of the model's predictions to it.
    \item Applying the derived metrics, we add to the evaluation of TFMs. Our results show that all current state-of-the-art TFMs violate both requirements on all datasets.
\end{itemize}

\section{Tabular Foundation Models}\label{sec:tfms}

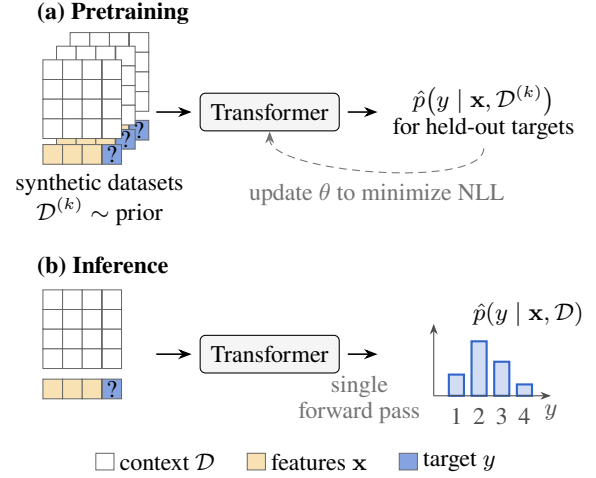
\begin{figure}[t]
    \centering
    \definecolor{pfnaccent}{RGB}{51,102,204}
\definecolor{pfnfeat}{RGB}{230,159,0}
\begin{tikzpicture}[
    font=\small,
    cell/.style={draw=black!60, line width=0.3pt, fill=white},
    featcell/.style={cell, fill=pfnfeat!30},
    ycell/.style={cell, fill=pfnaccent!60},
    tbox/.style={draw, rounded corners=2pt, align=center, inner sep=4pt, fill=black!4},
    flow/.style={->, >={Stealth[length=1.8mm]}, semithick},
    grad/.style={->, >={Stealth[length=1.6mm]}, densely dashed, black!55},
    lab/.style={align=center, font=\small},
]
\def\cs{0.26}%
\newcommand{\pfnctx}[4]{%
    \begin{scope}[shift={(#1,#2)}]
        \foreach \r in {1,...,#3}{%
            \foreach \c in {1,...,#4}{%
                \path[cell] ({(\c-1)*\cs},{-\r*\cs}) rectangle ++(\cs,\cs);
            }
            \path[cell] ({#4*\cs},{-\r*\cs}) rectangle ++(\cs,\cs);
        }
        \draw[black!60, line width=0.6pt] ({#4*\cs},0) -- ({#4*\cs},{-#3*\cs});
    \end{scope}
}
\newcommand{\pfnquery}[3]{%
    \begin{scope}[shift={(#1,#2)}]
        \foreach \c in {1,...,#3}{%
            \path[featcell] ({(\c-1)*\cs},{-\cs}) rectangle ++(\cs,\cs);
        }
        \path[ycell] ({#3*\cs},{-\cs}) rectangle ++(\cs,\cs);
        \draw[black!60, line width=0.6pt] ({#3*\cs},0) -- ({#3*\cs},{-\cs});
        \node at ({(#3+0.5)*\cs},{-0.5*\cs}) {?};
    \end{scope}
}

\begin{scope}
    \node[font=\small\bfseries, anchor=west] at (-0.15,1.35) {(a) Pretraining};
    \pfnctx{0.41}{1.10}{4}{3}
    \pfnquery{0.41}{-0.06}{3}
    \pfnctx{0.23}{0.92}{4}{3}
    \pfnquery{0.23}{-0.18}{3}
    \pfnctx{0.05}{0.74}{4}{3}
    \pfnquery{0.05}{-0.38}{3}
    \node[lab] at (0.80,-1.1) {synthetic datasets\\ $\ctx^{(k)} \sim$ prior};

    \node[tbox] (tfa) at (3.05,0.05) {Transformer};

    \draw[flow] (1.55,0.05) -- ([xshift=-3pt]tfa.west);
    \draw[flow] ([xshift=3pt]tfa.east) -- ([xshift=0.50cm]tfa.east);

    \node[lab, anchor=west] (outa) at ([xshift=0.58cm]tfa.east)
        {$\hat{p}\bigl(y \mid \mathbf{x}; \ctx^{(k)}\bigr)$\\ for held-out targets};

    \draw[grad] (outa.south) .. controls +(0,-0.55) and +(0,-0.55) .. (tfa.south)
        node[midway, below=1pt, lab, black!55] {update $\theta$ to minimize NLL};
\end{scope}

\begin{scope}[shift={(0,-3.15)}]
    \node[font=\small\bfseries, anchor=west] at (-0.15,1.15) {(b) Inference};
    \pfnctx{0.05}{0.85}{4}{3}
    \pfnquery{0.05}{-0.33}{3}

    \node[tbox] (tfb) at (3.05,0.00) {Transformer};

    \draw[flow] (1.55,0.00) -- ([xshift=-3pt]tfb.west);
    \draw[flow] ([xshift=3pt]tfb.east) -- ([xshift=0.50cm]tfb.east);
    \node[lab, black!55, anchor=north] at ([xshift=0.26cm,yshift=-0.16cm]tfb.east)
        {single\\ forward pass};

    \begin{scope}[shift={(5.22,-0.55)}]
        \foreach \x/\h/\c in {0.30/0.28/1, 0.60/0.72/2, 0.90/0.45/3, 1.20/0.15/4}{%
            \fill[pfnaccent!22] ({\x-0.10},0) rectangle ({\x+0.10},\h);
            \draw[pfnaccent, line width=0.8pt] ({\x-0.10},0) rectangle ({\x+0.10},\h);
            \node[font=\small, black!70, below=1pt] at (\x,0) {\c};
        }
        \draw[->, >={Stealth[length=1.4mm]}, black!70] (0,0) -- (1.55,0)
            node[below=0pt, font=\small] {$y$};
        \draw[->, >={Stealth[length=1.4mm]}, black!70] (0,0) -- (0,0.95);
        \node[lab, anchor=south west] at (0.40,0.80) {$\hat{p}(y \mid \mathbf{x}; \ctx)$};
    \end{scope}

    \begin{scope}[shift={(0.75,-1.42)}]
        \path[cell] (0,-0.11) rectangle ++(0.24,0.22);
        \node[anchor=west, inner sep=0pt] (lgone) at (0.31,0) {context $\ctx$};
        \path[featcell] ($(lgone.east)+(0.42,-0.11)$) rectangle ++(0.24,0.22);
        \node[anchor=west, inner sep=0pt] (lgtwo) at ($(lgone.east)+(0.73,0)$) {features $\mathbf{x}$};
        \path[ycell] ($(lgtwo.east)+(0.42,-0.11)$) rectangle ++(0.24,0.22);
        \node[anchor=west, inner sep=0pt] at ($(lgtwo.east)+(0.73,0)$) {target $y$};
    \end{scope}
\end{scope}
\end{tikzpicture}
    \caption{The prior-data fitted network blueprint. (a)~A transformer is pretrained to predict held-out targets on synthetic datasets $\ctx^{(k)}$ drawn from a prior. (b)~With frozen weights, a single forward pass on a real training set $\ctx$ and test point $\mathbf{x}$ returns a univariate predictive density approximating the posterior predictive under that prior.}\label{fig:pfn}
\end{figure}

Tabular foundation models transfer the in-context learning paradigm to supervised prediction on tables. The standard blueprint is the prior-data fitted network (PFN;\@ Figure~\ref{fig:pfn}). A transformer is pretrained on millions of synthetic datasets drawn from a prior over data-generating processes (for instance, randomly sampled structural causal models) by predicting held-out target values from the remainder of each dataset. After pretraining, the model solves a new task without any gradient updates. It consumes a labeled training set $\ctx = {\{(\mathbf{x}_i, y_i)\}}_{i=1}^{n}$ as its context together with a test feature vector $\mathbf{x}$, and a single forward pass returns a predictive distribution $\hat{p}(y \mid \mathbf{x}; \ctx)$ over the target. The pretraining objective shapes this output into an approximation of the Bayesian posterior predictive distribution under the synthetic prior~\citep{Muller2022.Transformers}.

\textbf{One notation for both prediction tasks.}
TFMs serve two prediction tasks through two predictive heads: a categorical head over a finite set of classes, and a distributional head for probabilistic regression of a continuous value. The distributional head is realized as a bin-wise density~\citep{Hollmann2025.Accurate} or as a set of predicted quantiles~\citep{Qu2026.TabICLv2}, depending on the model. We treat classification and regression in a unified way. A target $y^{j}$ takes values in a target space $\Y^{j}$ that is either finite (classification) or a continuous subset of $\R$ (regression). $\hat{p}$ always denotes the density of the predictive distribution with respect to the base measure of that space: the counting measure in the finite case, where $\hat{p}$ is a probability mass function and $\int_{\Y^{j}}\!\cdots\,\mathrm{d}y^{j}$ is a finite sum, and the Lebesgue measure in the continuous case. We therefore speak of \emph{densities} and \emph{integrals} throughout, and every statement in this section and in Section~\ref{sec:consistencies} applies verbatim to classification targets, regression targets, and to a target pair of mixed types.

\textbf{Implied joint distributions.}
Like the univariate head of a TFM, a conditional neural process~\citep{Garnelo2018.Conditional} predicts every target independently given the context and cannot express dependencies between targets. \citet{Bruinsma2023.Autoregressive} obtain a joint without changing the model by deploying it autoregressively: sample one target, add it to the context, and predict the next target from the enlarged context. For targets $\mathbf{y} = (y^1, \dots, y^d)$ at inputs $\mathbf{x}^1, \dots, \mathbf{x}^d$ and any permutation $\pi$ of $\{1, \dots, d\}$, this defines the joint
\begin{equation}\label{eq:ar-joint}
    \phat{\pi}(\mathbf{y} \mid \mathbf{x}^{1:d}; \ctx)
    \;=\;
    \prod_{j=1}^{d} \hat{p}\bigl(y^{\pi(j)} \,\big|\, \mathbf{x}^{\pi(j)}; \ctx_{\pi(<j)}\bigr),
\end{equation}
where $\ctx_{\pi(<j)} \coloneqq \ctx \cup \{(\mathbf{x}^{\pi(k)}, y^{\pi(k)})\}_{k<j}$ is the context enlarged by the targets sampled so far. \citet{Vetter2025.Effortless} transfer this construction to TFMs, where the targets are different columns of the same row $\mathbf{x}$. Since the in-context interface is agnostic to which column acts as the target, a sampled value is appended as an additional feature column instead, to $\mathbf{x}$ and, using the observed values, to every row of $\ctx$.

\section{Consistency Properties}\label{sec:consistencies}
We adapt the consistency conditions of stochastic processes to the way TFMs are built, and formalize two requirements that are necessary to guarantee that the predicted posteriors stem from a single, well-defined multivariate distribution. We begin with the weaker of the two, and then show that one implies the other.

\textbf{Notation for a pair of targets.} Both requirements already apply for $d=2$, so we state them for two targets and write $A$ and $B$ for the two target columns, $a \in \Y^A$ and $b \in \Y^B$ for their values. Superscripts indicate target columns throughout: $\ctx^{A}$ is the context set in which column $A$ is the target and $B$ is one of the features, and $\ctx^{A}_{-B}$ is the same context with column $B$ removed. For $d>2$ both conditions apply verbatim to any pair of targets, with the remaining ones held fixed as features.

\textbf{Marginalization consistency.}
The first requirement ties the model's conditionals back to its direct univariate predictions: sampling the intermediate target from the model's own predictive distribution must not change what the model \emph{believes} about the remaining target. In the spirit of the Kolmogorov Extension Theorem~\citep{Oksendal2003.Stochastic}, marginalizing the factorization must recover the direct marginal.

\begin{definition}[Marginalization consistency]\label{def:marginalization-consistency}
A TFM is \emph{marginalization-consistent} iff for every $\mathbf{x}$ and all $a$,
\begin{equation}
    \hat{p}(a \mid \mathbf{x}; \ctx^{A}_{-B})
    \;=\;
    \int_{\Y^B} \hat{p}(a \mid b, \mathbf{x}; \ctx^{A})\, \hat{p}(b \mid \mathbf{x}; \ctx^{B}_{-A}) \,\mathrm{d}b.
    \tag{C1}\label{eq:marginalization-consistency}
\end{equation}
\end{definition}

A violation means that the model's marginal estimate of $a$ is
inconsistent with its own conditional predictions with respect to
$b$. If the predictions that enter~\eqref{eq:marginalization-consistency}
are the marginals and the conditional of one joint distribution of
$(a, b)$ given $\mathbf{x}$, the identity is its law of total probability.
The corresponding Kolmogorov condition for stochastic processes applies the law of
total probability with the context held fixed. It would require the direct
marginal to be $\hat{p}(a \mid \mathbf{x}; \ctx^{A})$. This is a prediction
whose context contains the training values of $B$ although $b$ is not part
of the test point. A TFM cannot be queried this way, since its context
always holds exactly the columns of the test point and the target, and
marking $b$ as missing is not the same as removing the column
(more details in Appendix~\ref{app:classical}).

\textbf{Factorization consistency.}
The second, stronger requirement concerns the constructed joint
distribution itself rather than its univariate projections. The chain
rule factorizes a joint density in either order, $p(a, b \mid
\mathbf{x}) = p(a \mid b, \mathbf{x})\, p(b \mid \mathbf{x}) = p(b
\mid a, \mathbf{x})\, p(a \mid \mathbf{x})$. If the model's
four predictive densities are the marginals and conditionals of one
joint, the two autoregressive constructions must therefore agree, i.e.,
$\phat{(a,b)} = \phat{(b,a)}$.

\begin{definition}[Factorization consistency]\label{def:factorization-consistency}
A TFM is \emph{factorization-consistent} iff for every $\mathbf{x}$ and $a \in \Y^A$, $b \in \Y^B$,
\begin{equation}
\begin{aligned}
    &\hat{p}(a \mid b, \mathbf{x}; \ctx^{A})\, \hat{p}(b \mid \mathbf{x}; \ctx^{B}_{-A}) \\
    &\qquad =\;
    \hat{p}(b \mid a, \mathbf{x}; \ctx^{B})\, \hat{p}(a \mid \mathbf{x}; \ctx^{A}_{-B})
\end{aligned}
    \tag{C2}\label{eq:factorization-consistency}
\end{equation}
\end{definition}

If this property is violated, the autoregressive approach depends on the order in which the targets are sampled. 
Equation~\eqref{eq:factorization-consistency} is a compatibility
condition, which characterizes when a family of conditionals and
marginals can arise from one joint distribution
\citep{Besag1974.Spatial,Arnold1989.Compatible}.

\textbf{The properties are strictly ordered.}\label{sec:hierarchy}
Although Definitions~\ref{def:marginalization-consistency} and~\ref{def:factorization-consistency} arise from different aspects of coherence, they are not independent. Whenever the predicted conditionals are properly normalized densities, factorization consistency already entails marginalization consistency.

\begin{proposition}[\eqref{eq:factorization-consistency} implies~\eqref{eq:marginalization-consistency}]\label{prop:implication}
If a model is factorization-consistent (Definition~\ref{def:factorization-consistency}), then it is also marginalization-consistent (Definition~\ref{def:marginalization-consistency}).
\end{proposition}

\begin{proof}
Let $\mathbf{x}$ be a feature vector and $a \in \Y^A$. We show that the right side of~\eqref{eq:marginalization-consistency}, the marginal density that the autoregressive construction assigns to $a$, reduces to $\hat{p}(a \mid \mathbf{x}; \ctx^{A}_{-B})$.
\begin{align*}
    &\int_{\Y^B} \hat{p}(a \mid b, \mathbf{x}; \ctx^{A})\, \hat{p}(b \mid \mathbf{x}; \ctx^{B}_{-A}) \,\mathrm{d}b \\
    &\quad\overset{\text{\eqref{eq:factorization-consistency}}}{=} \int_{\Y^B} \hat{p}(b \mid a, \mathbf{x}; \ctx^{B})\, \hat{p}(a \mid \mathbf{x}; \ctx^{A}_{-B}) \,\mathrm{d}b \\
    &\quad= \hat{p}(a \mid \mathbf{x}; \ctx^{A}_{-B}) \underbrace{\int_{\Y^B} \hat{p}(b \mid a, \mathbf{x}; \ctx^{B}) \,\mathrm{d}b}_{=\,1} \\
    &\quad= \hat{p}(a \mid \mathbf{x}; \ctx^{A}_{-B})
\end{align*}
The first step applies~\eqref{eq:factorization-consistency} to the integrand, the second pulls $\hat{p}(a \mid \mathbf{x}; \ctx^{A}_{-B})$ out of the integral, since it does not depend on $b$. The third step uses that $\hat{p}(\,\cdot \mid a, \mathbf{x}; \ctx^{B})$ is a normalized probability density and therefore integrates to one. Without the normalization,~\eqref{eq:marginalization-consistency} holds only up to some constant factor.
\end{proof}

By the contrapositive of~\Cref{prop:implication}, testing~\eqref{eq:marginalization-consistency} alone would suffice to detect violations. We nevertheless investigate both properties to quantify how strongly TFMs violate each.

\begin{proposition}[\eqref{eq:marginalization-consistency} does not imply~\eqref{eq:factorization-consistency}]\label{prop:no-converse}
Marginalization consistency does not imply factorization consistency.
\end{proposition}
We provide a counterexample in Appendix~\ref{app:no-converse}.

\textbf{Consistency matters.} The ground-truth data-generating
process that any TFM aims to capture \emph{is} a single
joint distribution $p^\star(a, b \mid \mathbf{x})$. Since $p^\star$ does not learn from
a context, its marginals and conditionals cohere whichever training
columns are available, and both~\eqref{eq:marginalization-consistency}
and~\eqref{eq:factorization-consistency} hold for it
(Appendix~\ref{app:data-generating}). A model that violates
either has therefore provably departed from $p^\star$, and the size of the
violation bounds this departure from below. At a test point
$\mathbf{x}$, let $\phat{(b,a)}$ and $\phat{(a,b)}$ be the left- and
right-hand sides of~\eqref{eq:factorization-consistency}, $\hat{m}_A$ the
right-hand side of~\eqref{eq:marginalization-consistency}, and $p^\star_A$
the marginal of $a$ under $p^\star$. With the total variation distance
$\mathrm{TV}(p, q) = \tfrac{1}{2}\int \lvert p - q \rvert$, the violations are
$\Delta_{\mathrm{M}} = \mathrm{TV}\bigl(\hat{p}(\cdot \mid \mathbf{x}; \ctx^{A}_{-B}), \hat{m}_A\bigr)$
and $\Delta_{\mathrm{F}} = \mathrm{TV}\bigl(\phat{(b,a)}, \phat{(a,b)}\bigr)$.

\begin{proposition}[Violations lower-bound the error]\label{prop:error-bound}
For every $\mathbf{x}$,
\begin{align*}
    \Delta_{\mathrm{F}}^2 &\le \mathrm{KL}(p^\star \,\|\, \phat{(a,b)}) + \mathrm{KL}(p^\star \,\|\, \phat{(b,a)}),\\
    \Delta_{\mathrm{M}}^2 &\le \mathrm{KL}\bigl(p^\star_A \,\|\, \hat{p}(\cdot \mid \mathbf{x}; \ctx^{A}_{-B})\bigr) + \mathrm{KL}(p^\star \,\|\, \phat{(b,a)}).
\end{align*}
\end{proposition}
The proof is deferred to Appendix~\ref{app:error-bound}. Since
$\mathrm{KL}(p^\star \,\|\, q)$ is the expected log-loss of $q$ in excess
of the entropy of $p^\star$, every violation lower-bounds the excess
held-out NLL of the model's predictions.
The incoherence also propagates into any downstream
task. Autoregressive data generation~\citep{Vetter2025.Effortless} draws records
from an arbitrary and order-dependent joint. Therefore, one cannot
trust these joints for decisions that are based on the dependence
between targets.

\section{Experiments}\label{sec:experiments}
We want to put the just-introduced conditions to the test on current TFMs. The goal is to detect and quantify to what extent each condition is violated. The proposed method shall establish a new aspect for analyzing and comparing TFMs and their prediction capabilities. Furthermore, we examine the relationship between the two consistency requirements.

\subsection{Setup}

\textbf{Models.} We conduct our experiments on several state-of-the-art TFMs: TabPFNv2/v3~\citep{Hollmann2025.Accurate,Grinsztajn2026.TabPFN3}, whose regression head returns a bin-wise piecewise-constant density, TabICLv2~\citep{Qu2026.TabICLv2}, whose head returns predicted quantiles, and \mbox{TabDPTv1.2}~\citep{Ma2025.TabDPT,Hosseinzadeh2026.TabDPTTurbo}, whose head is a softmax over equal-width bins of the standardized target space. For classification, we additionally include \mbox{TabICLv1}~\citep{Qu2025.TabICLa} and TabFM~\citep{Kong.TabFM}. While the former does not provide a regressor, the latter offers point regression only.

\textbf{Datasets.} We employ tabular datasets from OpenML~\citep{Feurer2021.OpenMLPython} in which two columns form a natural pair of targets that plausibly depend on each other given the remaining features, with all remaining columns serving as the feature vector $\mathbf{x}$. We evaluate on a broad suite spanning both classification and regression; Appendix~\ref{app:datasets} lists every dataset's target pair, size, feature counts, and OpenML ID, together with the selection criteria and preprocessing applied.

\textbf{Measuring disagreement.}
Each check compares two predictive distributions over the same target space, and we quantify their discrepancy with the total variation distance
\begin{equation}
    \mathrm{TV}(p,q)=\tfrac{1}{2}\sum_{c}\lvert p_c-q_c\rvert\in[0,1],\label{eq:tv}
\end{equation}
where the sum runs over the classes (classification) or over the cells of a discretization of the target space (regression). $\mathrm{TV}$ equals $0$ iff the distributions coincide and $1$ iff they are disjoint, and it can be read as follows: a value of $t$ means there is an event to which the two sides of the check assign probabilities differing by $t$. Because $\mathrm{TV}$ is bounded and unit-free, classification and regression results are reported on the same scale. 

The classification results can be considered \emph{exact}. The target spaces are finite, so both the marginalizing integral and the two chain-rule joints are finite sums that we evaluate in closed form from the model's categorical heads. Any nonzero $\mathrm{TV}$ is a violation by definition. 

For regression the random variables of interest are continuous, and therefore we cannot enumerate the target space. The distribution of targets modeled by current TFM implementations can be accessed by its quantile function called $\hat{F}^{-1}$. From it, we derive a bin-wise uniform approximation of the density. Consecutive quantile levels define bins of known probability mass, within which the mass is spread uniformly. Two predictive distributions can then be compared by $\mathrm{TV}$ over these bins as in~\eqref{eq:tv}; see Appendix~\ref{app:discretization} for details.
We run every check under $5$-fold cross-testing. 
All experiments were run on NVIDIA RTX 2080 Ti and A40 GPUs with 11GB and 48GB of memory, respectively.

\subsection{Marginalization Consistency}\label{sec:exp-marginalization}
\begin{table*}[t]
    \centering
    \small
    \setlength{\tabcolsep}{2pt}
    \resizebox{\textwidth}{!}{%
    \begin{tabular}{ll cccccc}
        \toprule
        & & TabPFNv2 & TabICLv1 & TabDPT & TabICLv2 & TabPFNv3 & TabFM \\
        \midrule
        \multicolumn{8}{c}{\textit{Marginalization consistency}~\eqref{eq:marginalization-consistency}} \\
        \midrule
        \multirow{10}{*}{\rotatebox{90}{Classification}}
        & \textsc{Anneal}*    & \second{\gap{0.0022}{.0003}} (0.28) & \worst{\gap{0.0089}{.0025}} (0.31) & \gap{0.0037}{.0011} \second{(0.24)}          & \gap{0.0029}{.0012} (0.29)          & \gap{0.0060}{.0018} \worst{(0.42)}          & \best{\gap{0.0010}{.0008}} \best{(0.18)} \\
        & \textsc{Credit}*    & \gap{0.0237}{.0023} \best{(0.12)}          & \worst{\gap{0.0282}{.0015}} (0.16) & \best{\gap{0.0218}{.0029}} (0.16)   & \second{\gap{0.0234}{.0023}} \second{(0.13)} & \gap{0.0273}{.0024} (0.15)          & \gap{0.0248}{.0038} \worst{(0.19)}        \\
        & \textsc{Phishing}*  & \gap{0.0338}{.0019} \worst{(0.43)}          & \worst{\gap{0.0339}{.0015}} \worst{(0.43)} & \best{\gap{0.0261}{.0024}} \best{(0.24)}   & \gap{0.0293}{.0021} (0.35)          & \second{\gap{0.0278}{.0026}} \second{(0.32)} & \gap{0.0285}{.0023} (0.33)        \\
        & \textsc{MIC}*       & \worst{\gap{0.0133}{.0018}} (0.16)  & \gap{0.0129}{.0008} (0.14)         & \second{\gap{0.0091}{.0011}} \best{(0.06)} & \gap{0.0096}{.0009} \second{(0.08)}          & \second{\gap{0.0091}{.0020}} (0.16) & \best{\gap{0.0090}{.0013}} \worst{(0.20)} \\
        & \textsc{Customer}*  & \second{\gap{0.0060}{.0002}} \best{(0.04)} & \worst{\gap{0.0077}{.0004}} (0.06) & \gap{0.0062}{.0011} \second{(0.05)}          & \second{\gap{0.0060}{.0009}} \worst{(0.09)} & \best{\gap{0.0057}{.0005}} \best{(0.04)}   & \gap{0.0061}{.0010} (0.07)        \\
        & \textsc{Car}       & \gap{0.0513}{.0046} (0.45)          & \worst{\gap{0.0621}{.0038}} \second{(0.39)} & \gap{0.0527}{.0021} \worst{(0.55)}          & \second{\gap{0.0506}{.0049}} (0.50) & \second{\gap{0.0506}{.0055}} (0.54) & \best{\gap{0.0171}{.0024}} \best{(0.18)} \\
        & \textsc{Marketing}* & \worst{\gap{0.0131}{.0013}} \second{(0.19)}  & \gap{0.0129}{.0009} (0.24)         & \best{\gap{0.0087}{.0006}} \best{(0.09)}   & \gap{0.0121}{.0011} (0.23)          & \second{\gap{0.0090}{.0016}} \worst{(0.34)} & \gap{0.0104}{.0009} (0.22)        \\
        & \textsc{Wine}*      & \worst{\gap{0.0478}{.0019}} (0.52)  & \gap{0.0344}{.0008} (0.44)         & \second{\gap{0.0157}{.0002}} \best{(0.39)} & \gap{0.0191}{.0007} \worst{(0.72)}          & \gap{0.0212}{.0006} (0.66)          & \best{\gap{0.0092}{.0008}} \second{(0.43)} \\
        & \textsc{Nursery}   & \gap{0.0536}{.0053} (0.42)          & \worst{\gap{0.0588}{.0030}} \second{(0.27)} & \gap{0.0546}{.0019} (0.43)          & \gap{0.0332}{.0010} \worst{(0.44)}          & \second{\gap{0.0231}{.0009}} (0.41) & \best{\gap{0.0137}{.0007}} \best{(0.25)} \\
        & \textsc{Diamonds}  & \second{\gap{0.0228}{.0007}} \best{(0.35)} & \worst{\gap{0.0339}{.0004}} \second{(0.37)} & \best{\gap{0.0169}{.0003}} (0.41)   & \gap{0.0250}{.0003} (0.47)          & \gap{0.0282}{.0006} \second{(0.37)}          & \gap{0.0312}{.0007} \worst{(0.76)}        \\
        \midrule
        \multirow{7}{*}{\rotatebox{90}{Regression}}
        & \textsc{Boston}    & \gap{0.0742}{.0082} (0.34)          & ---                        & \best{\gap{0.0597}{.0076}} \best{(0.29)}   & \worst{\gap{0.2850}{.0419}} \worst{(0.79)}  & \second{\gap{0.0623}{.0098}} \second{(0.33)} & ---                               \\
        & \textsc{QSAR}*      & \worst{\gap{0.0754}{.0054}} \worst{(0.39)}  & ---                        & \best{\gap{0.0584}{.0050}} (0.33)   & \gap{0.0710}{.0039} \second{(0.29)}          & \second{\gap{0.0651}{.0069}} \best{(0.28)} & ---                               \\
        & \textsc{Stock}     & \gap{0.0775}{.0063} \second{(0.31)}          & ---                        & \best{\gap{0.0313}{.0051}} \best{(0.25)}   & \second{\gap{0.0745}{.0056}} (0.32) & \worst{\gap{0.0777}{.0075}} \worst{(0.40)}  & ---                               \\
        & \textsc{Concrete}  & \worst{\gap{0.1398}{.0078}} \worst{(0.76)}  & ---                        & \best{\gap{0.0911}{.0078}} \second{(0.60)}   & \gap{0.1237}{.0060} (0.66)          & \second{\gap{0.1030}{.0084}} \best{(0.50)} & ---                               \\
        & \textsc{Insurance}* & \gap{0.0570}{.0016} \second{(0.34)}          & ---                        & \worst{\gap{0.0681}{.0030}} (0.37)  & \second{\gap{0.0559}{.0046}} \worst{(0.50)} & \best{\gap{0.0163}{.0029}} \best{(0.26)}   & ---                               \\
        & \textsc{Fiat}*      & \gap{0.0566}{.0107} \best{(0.34)}          & ---                        & \best{\gap{0.0524}{.0029}} \second{(0.36)}   & \worst{\gap{0.0723}{.0039}} \worst{(0.48)}  & \second{\gap{0.0529}{.0054}} \worst{(0.48)} & ---                               \\
        & \textsc{Kin8nm}    & \second{\gap{0.0815}{.0021}} \second{(0.32)} & ---                        & \best{\gap{0.0748}{.0016}} \best{(0.31)}   & \worst{\gap{0.0864}{.0021}} (0.33)  & \gap{0.0848}{.0026} \worst{(0.37)}          & ---                               \\
        
        \midrule
        \multicolumn{8}{c}{\textit{Factorization consistency}~\eqref{eq:factorization-consistency}} \\
        \midrule
        \multirow{10}{*}{\rotatebox{90}{Classification}}
        & \textsc{Anneal}*    & \gap{0.0148}{.0015} (0.29)          & \worst{\gap{0.0261}{.0015}} (0.36) & \second{\gap{0.0123}{.0020}} \second{(0.27)} & \gap{0.0129}{.0021} (0.29)          & \gap{0.0173}{.0023} \worst{(0.42)}          & \best{\gap{0.0042}{.0014}} \best{(0.19)} \\
        & \textsc{Credit}*    & \best{\gap{0.0470}{.0037}} \best{(0.14)}   & \worst{\gap{0.0632}{.0046}} (0.22) & \gap{0.0510}{.0020} (0.24)          & \gap{0.0517}{.0048} \worst{(0.26)}          & \gap{0.0506}{.0045} \second{(0.16)}          & \second{\gap{0.0484}{.0044}} (0.18) \\
        & \textsc{Phishing}*  & \gap{0.0522}{.0030} \worst{(0.43)}          & \worst{\gap{0.0536}{.0014}} \worst{(0.43)} & \gap{0.0436}{.0024} \best{(0.29)}          & \gap{0.0448}{.0031} (0.36)          & \second{\gap{0.0422}{.0022}} \second{(0.32)} & \best{\gap{0.0409}{.0032}} (0.34) \\
        & \textsc{MIC}*       & \gap{0.0277}{.0026} (0.16)          & \worst{\gap{0.0297}{.0018}} \worst{(0.24)} & \best{\gap{0.0162}{.0015}} \best{(0.13)}   & \gap{0.0254}{.0006} (0.16)          & \gap{0.0209}{.0017} \second{(0.15)}          & \second{\gap{0.0166}{.0014}} (0.20) \\
        & \textsc{Customer}*  & \gap{0.0231}{.0016} \worst{(0.12)}          & \worst{\gap{0.0252}{.0007}} (0.11) & \second{\gap{0.0180}{.0032}} \second{(0.09)} & \gap{0.0184}{.0016} \second{(0.09)}          & \gap{0.0185}{.0025} \best{(0.08)}          & \best{\gap{0.0144}{.0019}} (0.10) \\
        & \textsc{Car}       & \gap{0.0893}{.0085} (0.45)          & \gap{0.0939}{.0060} \second{(0.39)}         & \worst{\gap{0.1658}{.0374}} \worst{(0.55)}  & \gap{0.0644}{.0058} (0.50)          & \second{\gap{0.0616}{.0048}} (0.54) & \best{\gap{0.0216}{.0038}} \best{(0.18)} \\
        & \textsc{Marketing}* & \gap{0.0434}{.0038} \best{(0.20)}          & \worst{\gap{0.0467}{.0025}} \worst{(0.42)} & \gap{0.0312}{.0059} (0.39)          & \gap{0.0344}{.0040} (0.39)          & \second{\gap{0.0258}{.0059}} (0.36) & \best{\gap{0.0219}{.0013}} \second{(0.22)} \\
        & \textsc{Wine}*      & \worst{\gap{0.0487}{.0018}} (0.59)  & \gap{0.0350}{.0007} (0.48)         & \second{\gap{0.0165}{.0001}} \best{(0.41)} & \gap{0.0193}{.0007} \worst{(0.72)}          & \gap{0.0214}{.0006} (0.66)          & \best{\gap{0.0093}{.0008}} \second{(0.43)} \\
        & \textsc{Nursery}   & \gap{0.0745}{.0053} \worst{(0.45)}          & \gap{0.0746}{.0025} \second{(0.28)}         & \worst{\gap{0.1120}{.0043}} (0.43)  & \gap{0.0452}{.0019} (0.44)          & \second{\gap{0.0378}{.0020}} (0.41) & \best{\gap{0.0138}{.0007}} \best{(0.25)} \\
        & \textsc{Diamonds}  & \gap{0.0655}{.0005} \best{(0.38)}          & \worst{\gap{0.0916}{.0006}} \second{(0.46)} & \gap{0.0724}{.0006} (0.72)          & \gap{0.0744}{.0010} (0.61)          & \second{\gap{0.0632}{.0010}} (0.54) & \best{\gap{0.0593}{.0004}} \worst{(0.76)} \\
        \midrule
        \multirow{7}{*}{\rotatebox{90}{Regression}}
        & \textsc{Boston}    & \best{\gap{0.1335}{.0099}} \best{(0.36)}   & ---                        & \gap{0.1426}{.0091} \second{(0.37)}          & \worst{\gap{0.4762}{.0257}} \worst{(0.90)}  & \second{\gap{0.1379}{.0061}} (0.39) & ---                               \\
        & \textsc{QSAR}*      & \best{\gap{0.2068}{.0166}} \best{(0.62)}   & ---                        & \second{\gap{0.2116}{.0110}} \worst{(0.90)} & \gap{0.2181}{.0058} (0.74)          & \worst{\gap{0.2284}{.0055}} \second{(0.73)}  & ---                               \\
        & \textsc{Stock}     & \best{\gap{0.1307}{.0053}} \worst{(0.46)}   & ---                        & \second{\gap{0.1448}{.0081}} \second{(0.40)} & \worst{\gap{0.1626}{.0054}} \best{(0.38)}  & \gap{0.1484}{.0043} (0.44)          & ---                               \\
        & \textsc{Concrete}  & \gap{0.2206}{.0115} (0.77)          & ---                        & \second{\gap{0.1815}{.0085}} \second{(0.74)} & \worst{\gap{0.2578}{.0118}} \worst{(0.80)}  & \best{\gap{0.1813}{.0067}} \best{(0.68)}   & ---                               \\
        & \textsc{Insurance}* & \second{\gap{0.1634}{.0044}} \best{(0.53)} & ---                        & \worst{\gap{0.2576}{.0115}} \second{(0.73)}  & \gap{0.2366}{.0093} \worst{(0.78)}          & \best{\gap{0.1620}{.0100}} (0.74)   & ---                               \\
        & \textsc{Fiat}*      & \second{\gap{0.1617}{.0042}} \best{(0.44)} & ---                        & \worst{\gap{0.1973}{.0040}} \worst{(0.87)}  & \gap{0.1883}{.0069} \second{(0.60)}          & \best{\gap{0.1464}{.0044}} (0.62)   & ---                               \\
        & \textsc{Kin8nm}    & \gap{0.2007}{.0068} \worst{(0.53)}          & ---                        & \second{\gap{0.1845}{.0027}} \worst{(0.53)} & \worst{\gap{0.2011}{.0028}} \second{(0.50)}  & \best{\gap{0.1570}{.0028}} \best{(0.43)}   & ---                               \\
        
        \bottomrule
    \end{tabular}}
    \caption{Consistency violations, one column per model. \emph{Top:} marginalization consistency~\eqref{eq:marginalization-consistency}, the total variation distance $\mathrm{TV}\in[0,1]$ between the direct marginal $\hat{p}(a\mid\mathbf{x};\ctx^{A}_{-B})$ and the marginalized mixture. \emph{Bottom:} factorization consistency~\eqref{eq:factorization-consistency}, the TV between the two chain-rule joints. Values are averaged over the test set, with the standard deviation over the five cross-testing folds as a subscript and the worst single test instance in parentheses. Larger is worse; a consistent model would score $0$. Per row, \best{green (bold)} and \second{blue} mark the smallest and second-smallest mean gap, \worst{red} the largest; the parenthesized worst-instance values are ranked and colored the same way. TabICLv1 ships no regressor and TabFM's regression head emits only point predictions; both are evaluated on classification only. An asterisk (*) marks datasets included in the TabArena benchmark.}\label{tab:consistency}
\end{table*}

Condition~\eqref{eq:marginalization-consistency} asks whether marginalizing the autoregressive factorization over the intermediate target $b$ recovers the directly predicted marginal of $a$. We form the mixture $\int_{\Y^B}\hat{p}(a\mid b,\mathbf{x};\ctx^{A})\,\hat{p}(b\mid\mathbf{x};\ctx^{B}_{-A})\,\mathrm{d}b$ from the model's own conditional and marginal heads and compare it against $\hat{p}(a\mid\mathbf{x};\ctx^{A}_{-B})$.

\textbf{Classification.} For categorical targets, the integral is a finite sum, so we evaluate the conditional $\hat{p}(a\mid b,\mathbf{x};\ctx^{A})$ at every class of $b$, weight each by the model's predicted $\hat{p}(b\mid\mathbf{x};\ctx^{B}_{-A})$, and compare the resulting mixture to the direct marginal. Figure~\ref{fig:wine-teaser} already illustrated such a violation on a single \textsc{Wine} instance. 

\emph{Observations.}
The top panel of \Cref{tab:consistency} shows that the violations are systematic. Every model violates the condition on every dataset. Since the mixture is evaluated in closed form, these gaps cannot be sampling noise. The detour through $b$ alters the model's estimate of $a$, even though the value of $b$ is drawn from the TFM's own prediction and adds no information about the test point. No single model has the smallest violations on every dataset, so there is no uniformly \emph{most consistent} classifier for marginalization: TabFM has the lowest average TV on five of the ten datasets, TabDPT on four, and TabPFNv3 on the remaining one. TabICLv1 turns out to be the most inconsistent model.

\textbf{Regression.} We cannot enumerate $\Y^B$, so we discretize it deterministically: we replace $\hat{p}(b\mid\mathbf{x};\ctx^{B}_{-A})$ by $K=1{,}000$ equal-mass atoms $b_k={(\hat{F}^{B})}^{-1}\big(\tfrac{k-1/2}{K}\big)$ taken from the quantile head. We compute a midpoint quadrature of the marginalization integral and average the conditional \emph{CDFs}, $\hat{F}_{\mathrm{mix}}=\tfrac{1}{K}\sum_{k}\hat{F}(\,\cdot\mid b_k,\mathbf{x};\ctx^{A})$. The averaged CDF is compared to the direct marginal via $\mathrm{TV}$ on a shared $S=20$-cell grid.

The resolutions $K=1{,}000$ and $S=20$ are deliberately set to opposite
ends. A large $K$ makes the quadrature of the marginalization integral
as precise as needed, so the mixture is not distorted by a coarse
approximation of $\hat{p}(b\mid\mathbf{x})$. The comparison grid is coarsened for the opposite reason: because $\mathrm{TV}$ is non-increasing under coarsening~\citep{Csiszar2004.Information}, merging cells can only hide discrepancies that fall within a cell but never create them. An $S=20$-cell grid is coarse enough for the comparison to be insensitive to the quadrature discretization. The resulting value for $\mathrm{TV}$ is therefore a \emph{lower bound} for any evaluation on a finer grid, and in particular for the $\mathrm{TV}$ between the underlying continuous distributions~\citep{Nielsen2018.GUARANTEED}.

\emph{Observations.} In the regression setting there is not a single TFM that scores an average TV close to zero on any dataset. Again, TabDPT appears to be the model that comes closest to marginalization consistency, as it has the smallest TV on all datasets but \textsc{Insurance}. There, however, it has the highest TV among tested models.

\subsection{Factorization Consistency}\label{sec:exp-factorization}
Condition~\eqref{eq:factorization-consistency} is the stronger requirement. It asks whether the joint distribution built by the chain rule is independent of the order in which the two targets are sampled. We construct both joints and measure how far apart they are.

\textbf{Classification.} We build each joint exactly, enumerating over all possible $a,b$ for each of the two orders ($\hat{p}(a\mid b,\mathbf{x};\ctx^{A})$, $\hat{p}(b\mid a,\mathbf{x};\ctx^{B})$). The resulting joint distributions are therefore tables of size $|\Y^{A}| \times |\Y^{B}|$, where each cell represents a probability for the combination of some $a$ and $b$. The two tables resulting from each order are then compared based on TV\@.

\emph{Observations.} 
The bottom panel of \Cref{tab:consistency} shows that the two orders disagree everywhere, and by a wider margin than the marginalization check. This translates to the intuition that follows from \Cref{prop:implication}: factorization consistency is \emph{harder} to achieve than marginalization consistency.
TabFM stands out as the TFM that comes closest to fulfilling factorization consistency, as it has the lowest average TV on eight of the ten datasets and comes in second on the remaining two. Similar to the marginalization consistency experiment, TabICLv1 is the classifier with the highest violations on average. 

\textbf{Regression.} The joint cannot be enumerated, so we discretize each order onto the same $128\times128$ grid of $(a,b)$ cells whose axes span the range between the outermost quantiles. For order $b\to a$, the $b$-axis cell masses and each cell's probability-mass midpoint $b^{*}_{l}$ come from the quantile head of $\hat{p}(b\mid\mathbf{x};\ctx^{B}_{-A})$. Then, the conditional $\hat{p}(a\mid b^{*}_{l},\mathbf{x};\ctx^{A})$ distributes each cell's mass over the $a$-axis through its CDF\@. This is again evaluated from the quantile head. The converse order $a\to b$ is constructed analogously.

\emph{Observations.}
The order-dependence carries over to the continuous setting. The two chain-rule orders thus disagree not only on the marginals but on the dependence structure of the constructed joint, so the \emph{joint distribution} produced by the autoregressive workaround depends on the order in which the targets happen to be sampled. 

In contrast to the marginalization consistency experiment, TabDPT belongs to the \emph{more inconsistent} models in this setting. TabPFNv3 is the regressor that comes closest to factorization consistency among the evaluated TFMs, as it has the lowest average TV on the majority of datasets.

\subsection{Discussion}

\begin{figure}[t]
    \centering
    \includegraphics[width=\columnwidth]{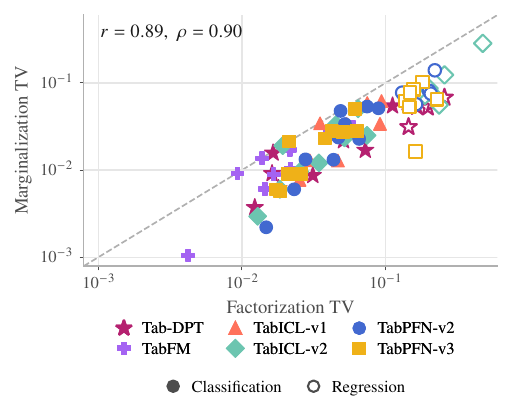}
    \caption{Factorization- versus marginalization-consistency TV, one marker per (model, dataset) pair, filled for classification and open for regression; the dashed line marks equality. The two gaps are strongly associated (Pearson $r=0.89$, Spearman $\rho=0.90$).}\label{fig:factorization-vs-marginalization}
\end{figure}

Every model violates both conditions on every dataset, so the observed inconsistency is a structural property of current TFMs. TabFM, the newest and top-ranked model on TabArena~\citep{Erickson2026.TabArena}, comes closest to being consistent among all tested classifiers. Based on that, one could assume that a model's consistency automatically improves with its predictive performance. However, relating each model's dataset-level accuracy or RMSE to its consistency gap (Appendix~\ref{app:acc-tv}) reveals no strong correlation. Therefore, consistency is a distinct axis that benchmarks such as TabArena do not capture.

We find a strong association between the two consistency types (see Figure~\ref{fig:factorization-vs-marginalization}).
Additionally, for the classification task we find that the TV measures for the two consistency requirements are higher for datasets with fewer features. However, we do not observe this for regression (see Appendix~\ref{app:tv-features}).

TabDPT is a revealing case in two aspects. First, it is trained with a self-supervised objective that predicts each column from the others~\citep{Ma2025.TabDPT,Hosseinzadeh2026.TabDPTTurbo}. As the roles of target and feature are repeatedly swapped during training, one would intuitively expect it to align the two conditional directions and thus improve factorization consistency. This expectation is not met: TabDPT is one of the more inconsistent models under factorization, while it is relatively consistent under marginalization.
Second, unlike the other models, TabDPT is trained not on synthetic but on real data, a corpus that includes those standard OpenML tables in our suite that are not drawn from the curated TabArena re-uploads. Yet it is not measurably more consistent on datasets seen during training than on the curated ones it has not seen. TabArena datasets are marked with an $*$ in \Cref{tab:consistency}. 

\section{Related Work}\label{sec:related-work}
\textbf{Coherence of predictive systems.}
That a sequence of one-step-ahead predictions defines a valid joint is the defining property of a coherent (martingale) predictive system. This is the foundation of the martingale-posterior view of Bayesian inference~\citep{Berti2004.Limit,Fong2023.Martingale}. \citet{Falck2024.InContext} use exactly this martingale property to test whether in-context learning in large language models (LLMs) is Bayesian. TabMGP builds martingale posteriors on top of TabPFN, implicitly assuming the requisite coherence~\citep{Ng2025.TabMGP}.

Our marginalization-consistency condition is a tabular, across-column analog of the martingale property these works rely on, adapted to the way TFMs are queried (Appendix~\ref{app:sp-to-tfm}). Its violation therefore does not refute the martingale property itself, but it shows that TFMs do not deliver coherence for free, so constructions that assume it need to verify it.

\textbf{Incoherence of any-order models.}
The tension between univariate training and joint deployment is well documented outside the tabular setting. Order-agnostic autoregressive models train a single network to model every conditional ordering~\citep{Uria2014.Deep,Hoogeboom2021.Autoregressive}, but the orderings do not necessarily agree. \citet{Shih2022.Training} show these models are internally redundant, defining the same distribution in several inconsistent ways.
LLMs likewise return probability judgments that violate the basic properties of probability~\citep{Zhu2024.Incoherent,Luo2025.Probability}. We bring this lens to tabular foundation models, and unlike these qualitative diagnoses, we derive metrics that quantify the extent of inconsistency.

\textbf{Kolmogorov consistency in neural processes.}
Neural processes are meta-learned predictors that map a context set to a predictive distribution, and PFNs are an instance of them~\citep{Muller2023.PFNs4BO}. It is known that they generally fail the Kolmogorov consistency conditions that are required of stochastic processes~\citep{Oksendal2003.Stochastic}. \citet{Young2026.Conditioning} precisely defines this for conditional neural processes~\citep{Garnelo2018.Conditional}, establishing a \emph{conditioning consistency gap} between adding a point to the context and conditioning on it. \citet{Bruinsma2023.Autoregressive} deploy conditional neural processes autoregressively, feeding sampled targets back into the context, which is how TFMs build multivariate joints. They point out that this gives up consistency under permutation and marginalization, study how the ordering of the target points affects the predictive log-likelihood, and recommend averaging over random orderings. However, they do not measure the discrepancy between the joints themselves. \citet{Mortimer2026.Incremental} measure how causal masking makes the predictions of transformer neural processes depend on the order in which observed points are streamed into the context, as a KL gap to the average over orderings.
Our conditions are the tabular counterpart of these gaps, with two differences. First, the Kolmogorov conditions for neural processes fix the context, which a TFM cannot do, since it cannot predict a marginal from a context that contains the marginalized column; our conditions compare predictions from the contexts a TFM can use (Appendix~\ref{app:sp-to-tfm}). Second, TFMs are deployed with the finite, and often small, context sizes at which such gaps are largest, and we measure the violation directly, as a distance between distributions, rather than bounding it asymptotically or through its effect on the log-likelihood.

\citet{Yalavarthi2026.Reliable} define marginalization consistency in the context of probabilistic forecasting for irregularly sampled multivariate time series. They show experimentally that a previous model fails to fulfill that condition, and introduce a consistent one.

\textbf{Compatibility of conditionals.}
Deciding whether a family of conditional and marginal distributions can arise from a common joint is a fundamental problem in statistics~\citep{Besag1974.Spatial,Arnold1989.Compatible,Arnold2001.Conditionally}. The same problem underlies dependency networks, where independently fitted conditionals are combined without any guarantee of coherence~\citep{Heckerman2000.Dependency}. Copulas offer the complementary, coherent-by-construction route to a joint~\citep{Nelsen2006.Introduction} by separating the marginals from the dependence structure. Our conditions reflect the compatibility requirements from the literature, and are turned into tests on a TFM's predictions.

\section{Conclusion}\label{sec:future-work}

TFMs become multivariate predictors by chaining univariate conditionals, a construction known not to yield a valid stochastic process for conditional neural processes. We asked whether this construction yields a valid stochastic process for TFMs, or informally, whether a TFM gives the same answer to the same question asked in two different ways. Since a TFM cannot be queried with a column in the context but not at the test point, we adapted the Kolmogorov conditions into marginalization and factorization consistency, which compare only predictions a TFM can make. As the data-generating distribution satisfies both requirements, every violation lower-bounds the model's distance to the ground truth. Every TFM we evaluate violates both requirements, on every dataset, for classification and regression. The answer to our question is therefore negative: TFMs give different answers to the same question asked in two different ways.

TabFM, which has the highest Elo on TabArena~\citep{Erickson2026.TabArena} among the evaluated models, comes closest to consistency among the classifiers. However, in general, we did not find that accuracy (Appendix~\ref{app:acc-tv}) leads to consistency. 
What we did find in our experiments is that the measures for marginalization and factorization consistency do correlate with each other.

Our results have an impact beyond the joint distributions of autoregressive constructions. 
For example, the fact that TFMs are not marginalization-consistent has practical implications for the handling of missing values. Imputing the missing value (\eqref{eq:marginalization-consistency}'s RHS) yields a different prediction than dropping the respective column from the context (\eqref{eq:marginalization-consistency}'s LHS).

For future work, one could investigate different strategies that yield TFMs that adhere to our consistency requirements. One possible direction is to penalize consistency violations during pretraining on the prior. Alternatively, one could try to construct TFM architectures that are restricted to predicting joint distributions in a way that induces a consistent stochastic process. While current TFMs only predict univariate distributions, future TFM versions could be equipped with a multivariate head. This would open up novel questions on the consistency of TFMs, relating to permutation invariance and marginalization consistency of the now multivariate target distributions.
However, our conditions remain relevant, as a multivariate TFM can also be chained autoregressively, and the autoregressive joints must match the predicted joint. 
In general, finding out whether consistency can be achieved without sacrificing expressivity is an intriguing open question for the ongoing research on TFMs.

\subsection*{AI Use Statement}

In this work, we used generative AI tools (coding assistants) for implementing and debugging the experiment code (model wrappers, consistency checks, result aggregation and plotting scripts). We have not used generative AI tools for generating data, designing the methodology, or interpreting the results. Additionally, we used generative AI tools for editing and polishing the text and for \LaTeX{} formatting. We have reviewed all AI-assisted work: all AI-assisted code was read, tested and checked for correctness by the authors, and the reported numbers were cross-checked against the raw outputs of the experiments. We take responsibility for the final content of this work, including text, claims, or artifacts produced with the aid of generative AI.

\bibliography{references}

\clearpage
\onecolumn
\aistatstitle{Do Tabular Foundation Models Agree with Themselves?\\
Supplementary Materials}

\appendix

\section{From Stochastic Processes to TFMs}\label{app:sp-to-tfm}

This appendix relates conditions~\eqref{eq:marginalization-consistency} and~\eqref{eq:factorization-consistency} to the consistency conditions of stochastic processes. We state the classical conditions and explain why a TFM cannot be queried for them (\Cref{app:classical}) and derive both conditions from the law of total probability (\Cref{app:derivation}). Furthermore, we show that \eqref{eq:marginalization-consistency} is the only form of marginalization consistency that constrains a TFM's own marginal predictions (\Cref{app:only-form}), and that the data-generating distribution satisfies both conditions (\Cref{app:data-generating}).

\subsection{The Classical Condition and the TFM Interface}\label{app:classical}

A stochastic process is a family of finite-dimensional distributions that is consistent under marginalization: the distribution of any set of variables is the marginal of the distribution of every larger set~\citep{Oksendal2003.Stochastic}. Neural processes~\citep{Garnelo2018.Conditional} carry this requirement over to predictions. For a fixed context, the prediction for a set of targets must be the marginal of the prediction for every larger set of targets. Transposed to the columns of a table, the context is the training set $\ctx$ with all of its columns, the targets are the test values $a$ and $b$, and the condition reads
\begin{equation}
    q(a \mid \mathbf{x}; \ctx)
    = \int_{\Y^B} q(a, b \mid \mathbf{x}; \ctx) \,\mathrm{d}b
    = \int_{\Y^B} q(a \mid b, \mathbf{x}; \ctx)\, q(b \mid \mathbf{x}; \ctx) \,\mathrm{d}b.
    \label{eq:classical-mc}
\end{equation}
All three predictions are conditioned on the same context $\ctx$, which contains the training values of both $A$ and $B$.
There can exist a neural process architecture that is able to produce each of the conditional distributions above; however,
a TFM is not.
Its context and its test point share their columns: every column of the context is either the target or a feature, and the test value of every feature is part of the test point. In particular, a TFM cannot receive a column $B$ in its context whose value $b$ is not present at the test point. For target $A$, it therefore admits exactly the two queries that enter~\eqref{eq:marginalization-consistency} and~\eqref{eq:factorization-consistency}. Either column $B$ is a feature in the context and $b$ is given at the test point, which is the conditional head $\hat{p}(a \mid b, \mathbf{x}; \ctx^{A})$, or column $B$ is removed from the context and $b$ from the test point, which is the direct head $\hat{p}(a \mid \mathbf{x}; \ctx^{A}_{-B})$. The left-hand side of~\eqref{eq:classical-mc} is neither, since it keeps the training values of $B$ but not $b$, and neither is the mixing distribution $q(b \mid \mathbf{x}; \ctx)$, which keeps the training values of $A$ but not $a$.

The closest a TFM gets to these predictions is to keep column $B$ in the context and set $b = \texttt{NaN}$ at the test point. The resulting prediction is not the marginal, though: in general,
\begin{equation*}
    \hat{p}(a \mid b = \texttt{NaN}, \mathbf{x}; \ctx^{A})
    \;\neq\;
    q(a \mid \mathbf{x}; \ctx)
    = \int_{\Y^B} q(a \mid b, \mathbf{x}; \ctx)\, q(b \mid \mathbf{x}; \ctx) \,\mathrm{d}b.
\end{equation*}
TFMs handle a \texttt{NaN} entry either by imputing it, which plugs in a single value of $b$ instead of integrating over all of them, or by treating the missingness as information, which conditions on the event that $b$ is missing. Neither integrates $b$ out under $q(b \mid \mathbf{x}; \ctx)$, so~\eqref{eq:classical-mc} cannot be evaluated with \texttt{NaN} entries either.

The same applies to factorization consistency. With the context held fixed, the chain rule reads
\begin{equation*}
    q(a \mid b, \mathbf{x}; \ctx)\, q(b \mid \mathbf{x}; \ctx)
    = q(b \mid a, \mathbf{x}; \ctx)\, q(a \mid \mathbf{x}; \ctx).
\end{equation*}
Its conditionals are the conditional heads, but its marginals $q(b \mid \mathbf{x}; \ctx)$ and $q(a \mid \mathbf{x}; \ctx)$ are the same predictions as in~\eqref{eq:classical-mc}. They are not TFM queries, and setting $a$ or $b$ to \texttt{NaN} does not produce them. Condition~\eqref{eq:factorization-consistency} replaces them by the direct heads, exactly as~\eqref{eq:marginalization-consistency} does for~\eqref{eq:classical-mc}.

\subsection{Derivation from the Law of Total Probability}\label{app:derivation}

For a pair of targets, a TFM answers four queries: the conditional heads $\hat{p}(a \mid b, \mathbf{x}; \ctx^{A})$ and $\hat{p}(b \mid a, \mathbf{x}; \ctx^{B})$, and the direct heads $\hat{p}(a \mid \mathbf{x}; \ctx^{A}_{-B})$ and $\hat{p}(b \mid \mathbf{x}; \ctx^{B}_{-A})$. Suppose that all four are the conditionals and marginals of one joint density $q(a, b \mid \mathbf{x})$:
\begin{equation}
\begin{aligned}
    q(a \mid b, \mathbf{x}) &= \hat{p}(a \mid b, \mathbf{x}; \ctx^{A}), &
    q(b \mid a, \mathbf{x}) &= \hat{p}(b \mid a, \mathbf{x}; \ctx^{B}), \\
    q(a \mid \mathbf{x}) &= \hat{p}(a \mid \mathbf{x}; \ctx^{A}_{-B}), &
    q(b \mid \mathbf{x}) &= \hat{p}(b \mid \mathbf{x}; \ctx^{B}_{-A}).
\end{aligned}
    \label{eq:single-joint}
\end{equation}
Then the law of total probability for $q$, $q(a \mid \mathbf{x}) = \int q(a \mid b, \mathbf{x})\, q(b \mid \mathbf{x}) \,\mathrm{d}b$, is~\eqref{eq:marginalization-consistency}, and the chain rule $q(a \mid b, \mathbf{x})\, q(b \mid \mathbf{x}) = q(b \mid a, \mathbf{x})\, q(a \mid \mathbf{x})$ is~\eqref{eq:factorization-consistency}. Both conditions are thus the classical ones for a joint that explains every prediction a TFM makes for the pair.

They differ from~\eqref{eq:classical-mc} only in where the marginals come from. Replacing each prediction in~\eqref{eq:classical-mc} by the TFM query with the same target and the same observed test columns yields~\eqref{eq:marginalization-consistency}, and this replacement is unique, because a TFM query is fixed by these two choices. The conditional $q(a \mid b, \mathbf{x}; \ctx)$ becomes the conditional head, which uses the same context, whereas the marginals $q(a \mid \mathbf{x}; \ctx)$ and $q(b \mid \mathbf{x}; \ctx)$ become the direct heads, which use the reduced contexts. For a joint $q(a, b \mid \mathbf{x}; \ctx)$ whose conditionals are the conditional heads, \eqref{eq:classical-mc} and~\eqref{eq:marginalization-consistency} therefore coincide if
\begin{equation}
    q(a \mid \mathbf{x}; \ctx) = \hat{p}(a \mid \mathbf{x}; \ctx^{A}_{-B})
    \quad\text{and}\quad
    q(b \mid \mathbf{x}; \ctx) = \hat{p}(b \mid \mathbf{x}; \ctx^{B}_{-A}),
    \label{eq:column-removal}
\end{equation}
that is, if removing the training values of a target that is unobserved at the test point leaves the prediction of the other target unchanged. Condition~\eqref{eq:column-removal} is not a TFM query either, since its left-hand sides are those of~\eqref{eq:classical-mc}.

\subsection{The Only Form That Constrains a TFM's Marginals}\label{app:only-form}

Neither~\eqref{eq:classical-mc} nor~\eqref{eq:column-removal} can be tested on a TFM. The following proposition shows that together they amount to~\eqref{eq:marginalization-consistency} and~\eqref{eq:factorization-consistency}, which can.

\begin{proposition}[Consistency and a single joint]\label{prop:single-joint}
Let all four heads be strictly positive.
\begin{enumerate}
    \item[(i)] The model satisfies~\eqref{eq:factorization-consistency} iff a joint $q(a, b \mid \mathbf{x})$ satisfying~\eqref{eq:single-joint} exists.
    \item[(ii)] Let the conditional heads be the conditionals of a strictly positive joint $q(a, b \mid \mathbf{x}; \ctx)$. The model satisfies~\eqref{eq:marginalization-consistency} for both targets iff $q$ satisfies~\eqref{eq:column-removal}, and the same holds for~\eqref{eq:factorization-consistency}.
\end{enumerate}
\end{proposition}

\begin{proof}
Fix $\mathbf{x}$ and suppress it together with the contexts. Write $q(a \mid b)$ and $q(b \mid a)$ for the conditional heads and $r(a) = \hat{p}(a \mid \mathbf{x}; \ctx^{A}_{-B})$ and $r(b) = \hat{p}(b \mid \mathbf{x}; \ctx^{B}_{-A})$ for the direct heads.

\emph{Part (i).} If $q$ satisfies~\eqref{eq:single-joint}, then~\eqref{eq:factorization-consistency} is its chain rule. Conversely, let~\eqref{eq:factorization-consistency} hold and set $J(a, b) = q(a \mid b)\, r(b)$. $J$ is a probability density, because $q(\cdot \mid b)$ integrates to one over $a$ and $r$ integrates to one over $b$, and its marginal in $b$ is $r(b)$. By~\eqref{eq:factorization-consistency}, $J(a, b) = q(b \mid a)\, r(a)$, so its marginal in $a$ is $r(a)$. Dividing $J$ by either marginal, both of which are strictly positive, returns the conditional heads, so $J$ satisfies~\eqref{eq:single-joint}.

\emph{Part (ii).} Write $q(a)$ and $q(b)$ for the marginals of $q$. Since $q$ is a joint with the conditional heads as its conditionals, the chain rule and the law of total probability give
\begin{equation}
    q(a \mid b)\, q(b) = q(a, b) = q(b \mid a)\, q(a),
    \qquad
    q(a) = \int_{\Y^B} q(a \mid b)\, q(b) \,\mathrm{d}b,
    \qquad
    q(b) = \int_{\Y^A} q(b \mid a)\, q(a) \,\mathrm{d}a.
    \label{eq:joint-identities}
\end{equation}
If $q$ satisfies~\eqref{eq:column-removal}, then $r(a) = q(a)$ and $r(b) = q(b)$, so~\eqref{eq:joint-identities} is~\eqref{eq:marginalization-consistency} for both targets and~\eqref{eq:factorization-consistency}. It remains to show the converse for each condition.

\emph{Factorization consistency.} Condition~\eqref{eq:factorization-consistency} reads $q(a \mid b)\, r(b) = q(b \mid a)\, r(a)$. Dividing it by the chain rule in~\eqref{eq:joint-identities}, whose sides are strictly positive, gives
\[
    \frac{r(b)}{q(b)} = \frac{r(a)}{q(a)} \qquad \text{for all } a \in \Y^A,\ b \in \Y^B.
\]
The left-hand side does not depend on $a$ and the right-hand side does not depend on $b$, so both equal a constant $c$. Integrating $r(b) = c\, q(b)$ over $\Y^B$ gives $c = 1$, since both densities are normalized. Hence $r(a) = q(a)$ and $r(b) = q(b)$.

\emph{Marginalization consistency for both targets.} Condition~\eqref{eq:marginalization-consistency} for $a$, and its counterpart for $b$ with the roles of the targets swapped, read $r(a) = \int q(a \mid b)\, r(b) \,\mathrm{d}b$ and $r(b) = \int q(b \mid a)\, r(a) \,\mathrm{d}a$. Subtracting the law of total probability in~\eqref{eq:joint-identities} shows that the differences $d(a) = r(a) - q(a)$ and $d(b) = r(b) - q(b)$ satisfy the same two equations. Substituting the second into the first gives
\[
    d(a') = \int_{\Y^A} T(a' \mid a)\, d(a) \,\mathrm{d}a,
    \qquad
    T(a' \mid a) = \int_{\Y^B} q(a' \mid b)\, q(b \mid a) \,\mathrm{d}b,
\]
where $T(\cdot \mid a)$ is a normalized density and $T > 0$ because $q$ is strictly positive. Therefore
\[
    \int_{\Y^A} \lvert d(a') \rvert \,\mathrm{d}a'
    = \int_{\Y^A} \Big\lvert \int_{\Y^A} T(a' \mid a)\, d(a) \,\mathrm{d}a \Big\rvert \,\mathrm{d}a'
    \le \int_{\Y^A} \int_{\Y^A} T(a' \mid a)\, \lvert d(a) \rvert \,\mathrm{d}a \,\mathrm{d}a'
    = \int_{\Y^A} \lvert d(a) \rvert \,\mathrm{d}a .
\]
Since the two ends coincide, the triangle inequality holds with equality, which requires that for almost every $a'$ the integrand $T(a' \mid a)\, d(a)$ does not change sign in $a$. As $T > 0$, $d$ itself does not change sign. But $d$ integrates to zero, because $r$ and $q$ are both normalized, so $d = 0$ almost everywhere. Then also $d(b) = \int q(b \mid a)\, d(a) \,\mathrm{d}a = 0$, hence $r(a) = q(a)$ and $r(b) = q(b)$.
\end{proof}

Proposition~\ref{prop:single-joint} settles which condition on a TFM expresses marginalization consistency. For a stochastic process, \eqref{eq:classical-mc} constrains the marginal that the process predicts. For a TFM, it constrains nothing the TFM predicts: whenever the conditional heads are compatible, the joint they determine~\citep{Arnold1989.Compatible} satisfies~\eqref{eq:classical-mc} by construction, whatever the direct heads say. A condition that involves the TFM's own marginal predictions has to identify them with the marginals of that joint, which is~\eqref{eq:column-removal}, and by part~(ii) this is exactly~\eqref{eq:marginalization-consistency} for both targets, or equivalently~\eqref{eq:factorization-consistency}. When the conditional heads are incompatible, no joint reproduces them, and by part~(i), \eqref{eq:factorization-consistency} fails. In this sense our conditions are the only form of marginalization consistency that applies to a TFM's own predictions. They are stricter than~\eqref{eq:classical-mc}, which asks for the law of total probability on a fixed context only, but~\eqref{eq:classical-mc} is not a condition a TFM can be held to.

For~\eqref{eq:marginalization-consistency}, both directions are needed: if $a$ and $b$ are independent under $q$, the condition for $a$ holds whatever the direct head of $b$ predicts. Our experiments test the condition for the main target only, which is necessary for consistency but not sufficient.

\subsection{The Data-Generating Distribution}\label{app:data-generating}

Our conditions are meant to hold for the target of learning, the data-generating distribution $p^\star(a, b \mid \mathbf{x})$. Used as a predictor, $p^\star$ answers each TFM query with the distribution of the target given the columns that the query contains. Removing a column from the table amounts to marginalizing it out of $p^\star$: if the rows of a table are drawn from $p^\star$, the same table without column $B$ has rows drawn from its marginal over $(\mathbf{x}, a)$. The answer to the direct query $\hat{p}(a \mid \mathbf{x}; \ctx^{A}_{-B})$ is therefore $p^\star(a \mid \mathbf{x}) = \int p^\star(a \mid b, \mathbf{x})\, p^\star(b \mid \mathbf{x}) \,\mathrm{d}b$, and the answer to the conditional query is $p^\star(a \mid b, \mathbf{x})$. Neither depends on the training values in the context, so~\eqref{eq:column-removal} holds trivially, and all four heads are the conditionals and marginals of the single joint $p^\star$. Condition~\eqref{eq:marginalization-consistency} is then its law of total probability and~\eqref{eq:factorization-consistency} its chain rule, so both hold exactly.

For $p^\star$, the difference between~\eqref{eq:classical-mc} and~\eqref{eq:marginalization-consistency} thus disappears: removing the column of an unobserved target and integrating that target out give the same distribution. A TFM that has learned $p^\star$ satisfies both of our conditions whatever context it is given, and every violation is a departure from $p^\star$, whose size Proposition~\ref{prop:error-bound} bounds from below.

\clearpage
\section{Deferred Proofs}\label{app:proofs}

\subsection{Marginalization Consistency Does Not Imply Factorization Consistency}\label{app:no-converse}

We give the construction behind Proposition~\ref{prop:no-converse}, which separates the two consistency properties by exhibiting predictive densities that satisfy marginalization consistency in both targets yet violate factorization consistency.

\newcounter{savedproposition}
\setcounter{savedproposition}{\value{proposition}}
\setcounter{proposition}{1}
\begin{proposition}[Marginalization consistency does not imply factorization consistency, restated]
There exist predictive densities that satisfy marginalization consistency in both targets yet violate factorization consistency.
\end{proposition}

\begin{proof}[Construction]
Fix $\mathbf{x}$ and take both targets binary, $a, b \in \{0, 1\}$. Let the marginals be uniform, $\hat{p}(a \mid \mathbf{x}) = \hat{p}(b \mid \mathbf{x}) = (\tfrac12, \tfrac12)$, and set the conditionals (rows indexed by the free target, columns by the conditioning target)
\[
    \hat{p}(a \mid b, \mathbf{x}) = \begin{pmatrix} \tfrac34 & \tfrac14 \\ \tfrac14 & \tfrac34 \end{pmatrix},
    \qquad
    \hat{p}(b \mid a, \mathbf{x}) = \begin{pmatrix} \tfrac12 & \tfrac12 \\ \tfrac12 & \tfrac12 \end{pmatrix},
\]
so that $b$ is independent of $a$ while $a$ depends on $b$. Both conditionals are properly normalized. Marginalizing $\hat{p}(a \mid b, \mathbf{x})$ against the uniform $\hat{p}(b \mid \mathbf{x})$ gives $\tfrac12(\tfrac34, \tfrac14) + \tfrac12(\tfrac14, \tfrac34) = (\tfrac12, \tfrac12) = \hat{p}(a \mid \mathbf{x})$, and marginalizing $\hat{p}(b \mid a, \mathbf{x})$ against $\hat{p}(a \mid \mathbf{x})$ trivially returns $(\tfrac12, \tfrac12) = \hat{p}(b \mid \mathbf{x})$; marginalization consistency holds in both targets. But at $(a, b) = (0, 0)$,
\begin{align*}
    \hat{p}(a{=}0 \mid b{=}0, \mathbf{x})\, \hat{p}(b{=}0 \mid \mathbf{x}) &= \tfrac34 \cdot \tfrac12 = \tfrac38, \\
    \hat{p}(b{=}0 \mid a{=}0, \mathbf{x})\, \hat{p}(a{=}0 \mid \mathbf{x}) &= \tfrac12 \cdot \tfrac12 = \tfrac14,
\end{align*}
so the two sides of~\eqref{eq:factorization-consistency} differ and factorization consistency fails. The two orders build joints with identical uniform marginals but different dependence: correlated one way, independent the other. Binary targets serve only to keep the arithmetic explicit; the same construction carries over to continuous target spaces, with the two-point densities replaced by mixtures of two disjointly supported components.
\end{proof}
\setcounter{proposition}{\value{savedproposition}}

The conditionals of this construction are incompatible, since no joint makes $b$ independent of $a$ while $a$ depends on $b$. This is necessary for any such construction: by Proposition~\ref{prop:single-joint}(ii), for conditionals of a strictly positive joint, marginalization consistency in both targets and factorization consistency are each equivalent to~\eqref{eq:column-removal}, and hence to each other.

\subsection{Violations Lower-Bound the Error}\label{app:error-bound}

We prove Proposition~\ref{prop:error-bound}, with $\Delta_{\mathrm{M}}$, $\Delta_{\mathrm{F}}$, $\phat{(b,a)}$, $\phat{(a,b)}$, $\hat{m}_A$ and $p^\star_A$ as defined in Section~\ref{sec:consistencies}.

\setcounter{savedproposition}{\value{proposition}}
\setcounter{proposition}{2}
\begin{proposition}[Violations lower-bound the error, restated]
For every $\mathbf{x}$,
\begin{align*}
    \Delta_{\mathrm{F}}^2 &\le \mathrm{KL}(p^\star \,\|\, \phat{(a,b)}) + \mathrm{KL}(p^\star \,\|\, \phat{(b,a)}),\\
    \Delta_{\mathrm{M}}^2 &\le \mathrm{KL}\bigl(p^\star_A \,\|\, \hat{p}(\cdot \mid \mathbf{x}; \ctx^{A}_{-B})\bigr) + \mathrm{KL}(p^\star \,\|\, \phat{(b,a)}).
\end{align*}
\end{proposition}
\setcounter{proposition}{\value{savedproposition}}

\begin{proof}
Fix $\mathbf{x}$ and drop it from the notation, and write $r_A = \hat{p}(\cdot\,; \ctx^{A}_{-B})$ for the direct head of $a$. We first bound both gaps by total variation distances to $p^\star$ and then pass to KL divergences.

\emph{Factorization in TV.} Adding and subtracting $p^\star$ inside the integral and applying $\lvert s + t \rvert \le \lvert s \rvert + \lvert t \rvert$ pointwise gives
\begin{equation}
\begin{aligned}
    \Delta_{\mathrm{F}}
    &= \tfrac12 \int_{\Y^A}\!\int_{\Y^B} \bigl\lvert \bigl(\phat{(b,a)}(a,b) - p^\star(a,b)\bigr) + \bigl(p^\star(a,b) - \phat{(a,b)}(a,b)\bigr) \bigr\rvert \,\mathrm{d}b\,\mathrm{d}a \\
    &\le \tfrac12 \int_{\Y^A}\!\int_{\Y^B} \bigl\lvert \phat{(b,a)}(a,b) - p^\star(a,b) \bigr\rvert \,\mathrm{d}b\,\mathrm{d}a
    + \tfrac12 \int_{\Y^A}\!\int_{\Y^B} \bigl\lvert \phat{(a,b)}(a,b) - p^\star(a,b) \bigr\rvert \,\mathrm{d}b\,\mathrm{d}a \\
    &= \mathrm{TV}(\phat{(b,a)}, p^\star) + \mathrm{TV}(\phat{(a,b)}, p^\star).
\end{aligned}
\label{eq:tv-factorization}
\end{equation}

\emph{Marginalization in TV.} The same argument on $\Y^A$ gives $\Delta_{\mathrm{M}} = \mathrm{TV}(r_A, \hat{m}_A) \le \mathrm{TV}(r_A, p^\star_A) + \mathrm{TV}(\hat{m}_A, p^\star_A)$. The second term is bounded by the corresponding joints: $\hat{m}_A$ is the marginal of $a$ under $\phat{(b,a)}$ and $p^\star_A$ is the marginal of $a$ under $p^\star$, and marginalizing does not increase $\mathrm{TV}$,
\begin{equation*}
    \mathrm{TV}(\hat{m}_A, p^\star_A)
    = \tfrac12 \int_{\Y^A} \Bigl\lvert \int_{\Y^B} \bigl(\phat{(b,a)}(a,b) - p^\star(a,b)\bigr)\,\mathrm{d}b \Bigr\rvert \,\mathrm{d}a
    \le \tfrac12 \int_{\Y^A}\!\int_{\Y^B} \bigl\lvert \phat{(b,a)}(a,b) - p^\star(a,b) \bigr\rvert \,\mathrm{d}b\,\mathrm{d}a
    = \mathrm{TV}(\phat{(b,a)}, p^\star).
\end{equation*}
Together,
\begin{equation}
    \Delta_{\mathrm{M}} \le \mathrm{TV}(r_A, p^\star_A) + \mathrm{TV}(\phat{(b,a)}, p^\star).
    \label{eq:tv-marginalization}
\end{equation}

\emph{From TV to KL.} Pinsker's inequality~\citep{Csiszar2004.Information} states $\mathrm{TV}(p, q) \le \sqrt{\mathrm{KL}(p \,\|\, q)/2}$, and since $\mathrm{TV}$ is symmetric, $2\,\mathrm{TV}(q, p^\star)^2 \le \mathrm{KL}(p^\star \,\|\, q)$ for every $q$. Squaring~\eqref{eq:tv-factorization} and~\eqref{eq:tv-marginalization} and using $(s + t)^2 \le 2(s^2 + t^2)$ gives
\begin{align*}
    \Delta_{\mathrm{F}}^2
    &\le 2\,\mathrm{TV}(\phat{(a,b)}, p^\star)^2 + 2\,\mathrm{TV}(\phat{(b,a)}, p^\star)^2
    \le \mathrm{KL}(p^\star \,\|\, \phat{(a,b)}) + \mathrm{KL}(p^\star \,\|\, \phat{(b,a)}),\\
    \Delta_{\mathrm{M}}^2
    &\le 2\,\mathrm{TV}(r_A, p^\star_A)^2 + 2\,\mathrm{TV}(\phat{(b,a)}, p^\star)^2
    \le \mathrm{KL}(p^\star_A \,\|\, r_A) + \mathrm{KL}(p^\star \,\|\, \phat{(b,a)}).
\end{align*}
If a divergence is infinite, the corresponding bound holds trivially.
\end{proof}

The intermediate bounds~\eqref{eq:tv-factorization} and~\eqref{eq:tv-marginalization} are tighter than those of the proposition. By~\eqref{eq:tv-factorization}, at every test point at least one sampling order is at least $\Delta_{\mathrm{F}}/2$ away from $p^\star$ in $\mathrm{TV}$. Exchanging the roles of $a$ and $b$ gives the analogous bounds for marginalization consistency in the other target.

\textbf{Excess log-loss.} For any joint $q$, the expected log-loss under the data-generating distribution decomposes as $\E_{p^\star}[-\log q(a, b)] = H(p^\star) + \mathrm{KL}(p^\star \,\|\, q)$, where $H$ is the (differential) entropy. $\mathrm{KL}(p^\star \,\|\, q)$ is therefore the expected held-out NLL of $q$ in excess of the irreducible $H(p^\star)$. Averaging the factorization bound over the two sampling orders lower-bounds their mean excess joint NLL by $\Delta_{\mathrm{F}}^2/2$. Averaged over test points, Jensen's inequality turns this into an excess of at least $\tfrac12\bigl(\E_{\mathbf{x}}[\Delta_{\mathrm{F}}]\bigr)^2$, so the mean gaps of \Cref{tab:consistency} translate directly into a dataset-level bound.

\textbf{Univariate predictions.} The joint in the marginalization bound is assembled from two univariate heads, $\phat{(b,a)}(a, b) = \hat{p}(a \mid b, \mathbf{x}; \ctx^{A})\, \hat{p}(b \mid \mathbf{x}; \ctx^{B}_{-A})$. The chain rule of the KL divergence splits its divergence from $p^\star$ along the same factorization, so with $p^\star_B$ the marginal of $b$ under $p^\star$,
\begin{equation*}
    \Delta_{\mathrm{M}}^2
    \le \mathrm{KL}\bigl(p^\star_A \,\|\, \hat{p}(\cdot \mid \mathbf{x}; \ctx^{A}_{-B})\bigr)
    + \mathrm{KL}\bigl(p^\star_B \,\|\, \hat{p}(\cdot \mid \mathbf{x}; \ctx^{B}_{-A})\bigr)
    + \E_{b \sim p^\star_B}\Bigl[\mathrm{KL}\bigl(p^\star(\cdot \mid b, \mathbf{x}) \,\|\, \hat{p}(\cdot \mid b, \mathbf{x}; \ctx^{A})\bigr)\Bigr].
\end{equation*}
Each term is the excess expected log-loss of one ordinary univariate prediction: the direct heads of $a$ and $b$, and the conditional head of $a$ at the true value of $b$. A marginalization violation therefore certifies, without any label, that at least one of these three predictions is at least $\Delta_{\mathrm{M}}^2/3$ away from $p^\star$ in KL divergence.

The bounds hold for any predictor, so every violation is an error with respect to $p^\star$.

The bounds also hold for the gaps we report. For classification they are computed exactly. For regression they are evaluated on a coarsened grid, which can only decrease $\mathrm{TV}$ (Section~\ref{sec:exp-marginalization}), so up to the discretization floor of at most $1.5\times10^{-3}$ (Appendix~\ref{app:discretization-error}) they do not exceed $\Delta_{\mathrm{M}}$ or $\Delta_{\mathrm{F}}$.

\clearpage
\section{Dataset Details}\label{app:datasets}

All datasets are public tables from OpenML~\citep{Feurer2021.OpenMLPython}, fetched through \texttt{fetch\_openml} with either the version or the numeric dataset ID pinned so that the exact table can be reproduced. Table~\ref{tab:datasets} lists, for each dataset, its size, the split of its feature columns into numerical and categorical, the two targets $a$ and $b$ used in the consistency checks, and its OpenML dataset ID\@.

Preprocessing is identical for every dataset and deliberately minimal: rows with missing entries are dropped, the two target columns are removed from the feature set, and every remaining column, numerical or categorical, is passed to the model as a feature $\mathbf{x}$. Apart from the categorical encoding described below, which one of the model APIs requires, no scaling, imputation or feature selection is applied beyond what each model performs internally, so the numbers in Table~\ref{tab:datasets} are the tables the models actually see. Two tables need an exception. \textsc{MIC}'s $112$ columns are missing so pervasively that no row is complete and dropping incomplete rows alone would empty the table; there we first drop every column missing more than $2\%$ of its entries (mostly laboratory values and pre-admission blood-pressure readings, which are missing by design rather than at random), and then drop the remaining incomplete rows, which leaves $53$ columns and $1{,}596$ of the $1{,}699$ rows. \textsc{Marketing} stores the customer enrollment date as free text, which no model in our comparison reads as a date; we parse it to the number of days since the earliest enrollment, turning $663$ opaque string levels into one ordered numerical feature. The train/test protocol is the one described in Section~\ref{sec:experiments}.

\begin{table*}[h]
    \centering
    \setlength{\tabcolsep}{2.5pt}
     {\footnotesize $^{\dagger}$ \textsc{Wine} merges the red-wine table (ID 40691, 1{,}599 rows) with the white-wine table (ID 40498, 4{,}898 rows); the color of the source table is the second target.\\
    $^{\ddagger}$ \textsc{MIC} is the one table in which no row is complete, so columns missing more than $2\%$ of their entries are dropped before incomplete rows are; $53$ of its $112$ columns and $1{,}596$ of its $1{,}699$ rows survive.\\
    $^{\S}$ \textsc{Marketing} stores the customer enrollment date as free text; it is counted as a numerical feature because it is parsed to days since the earliest enrollment.}
    \begin{tabular}{ll rrr ll r}
        \toprule
        & Dataset & Rows & Feat. & num./cat. & Target $a$ (main) & Target $b$ (second) & OpenML ID \\
        \midrule
        \multirow{10}{*}{\rotatebox{90}{Class.}}
        & \textsc{Anneal}   & 898      & 37 &  6 / 31 & \texttt{classes} (anneal grade), 5 cl. & \texttt{steel}, 8 cl. & 46906 \\
        & \textsc{Credit}   & 1{,}000  & 19 &  7 / 12 & \texttt{good\_or\_bad\_customer}, 2 cl. & \texttt{checking\_status}, 4 cl. & 46918 \\
        & \textsc{Phishing} & 1{,}353  &  8 &  0 / 8 & \texttt{WebsiteType}, 3 cl.       & \texttt{SFH}, 3 cl.    & 46963 \\
        & \textsc{MIC}      & 1{,}596  & 51 &  2 / 49 & \texttt{LET\_IS} (outcome), 8 cl. & \texttt{ASP\_S\_n}, 2 cl. & 46980$^{\ddagger}$ \\
        & \textsc{Customer} & 1{,}723  & 12 &  5 / 7 & \texttt{bad\_client\_target}, 2 cl. & \texttt{education}, 6 cl. & 46938 \\
        & \textsc{Car}      & 1{,}728  &  5 &  0 / 5 & \texttt{class} (acceptability), 4 cl. & \texttt{safety}, 3 cl. & 40975         \\
        & \textsc{Marketing}& 2{,}216  & 24 & 17 / 7 & \texttt{Response} (uptake), 2 cl. & \texttt{Marital\_Status}, 8 cl. & 46940$^{\S}$ \\
        & \textsc{Wine}     & 6{,}497  & 11 & 11 / 0 & \texttt{class} (quality), 8 cl.   & \texttt{color}, 2 cl.  & 40691$^{\dagger}$ \\
        & \textsc{Nursery}  & 12{,}960 &  7 &  0 / 7 & \texttt{class} (recommendation), 5 cl. & \texttt{health}, 3 cl. & 26          \\
        & \textsc{Diamonds} & 53{,}940 &  8 &  7 / 1 & \texttt{cut}, 5 cl.               & \texttt{color}, 7 cl.  & 42225            \\
        \midrule
        \multirow{7}{*}{\rotatebox{90}{Regr.}}
        & \textsc{Boston}   & 506      & 12 & 10 / 2 & \texttt{MEDV}      & \texttt{RM}, $\rho=0.70$        & 531   \\
        & \textsc{QSAR}     & 907      &  5 &  5 / 0 & \texttt{LC50}      & \texttt{MLOGP}, $\rho=0.65$     & 46954 \\
        & \textsc{Stock}    & 950      &  8 &  8 / 0 & \texttt{company10} & \texttt{company1}, $\rho=0.71$  & 223   \\
        & \textsc{Concrete} & 1{,}030  &  7 &  7 / 0 & \texttt{strength}  & \texttt{water}, $\rho=-0.29$    & 44959 \\
        & \textsc{Insurance}& 1{,}338  &  5 &  2 / 3 & \texttt{charges}   & \texttt{bmi}, $\rho=0.20$       & 46931 \\
        & \textsc{Fiat}     & 1{,}538  &  6 &  5 / 1 & \texttt{price}     & \texttt{km}, $\rho=-0.86$       & 46907 \\
        & \textsc{Kin8nm}   & 8{,}192  &  7 &  7 / 0 & \texttt{y}         & \texttt{theta1}, $\rho=-0.14$   & 189   \\
        \bottomrule
    \end{tabular}   
    \caption{The OpenML datasets used in the consistency checks. \emph{Rows} is the number of rows after dropping incomplete records, \emph{Feat.} the number of feature columns (all columns except the two targets), split into numerical and categorical. Target $a$ is the main target whose predictive distribution the checks compare; target $b$ is the second target that is marginalized over, conditioned on, or sampled first. For classification targets we give the number of classes, for regression targets the Pearson correlation $\rho$ between the two targets.}\label{tab:datasets}

\end{table*}

\textbf{Categorical features.} The models differ in the input format their public interfaces accept, so a categorical column does not reach all of them by the same route. TabPFN and TabICL take a data frame and infer which columns are categorical from the column types, so we hand them the table unchanged and leave the encoding to whatever each model does internally. TabDPT accepts only a numerical matrix (its \texttt{fit} rejects anything else, and it applies mean imputation and standardization to every column it is given), and it offers no way to declare a column categorical, so for TabDPT we integer-encode the categorical columns ourselves, as the TabDPT evaluation does~\citep{Ma2025.TabDPT}. The level order is fixed on the training split and reused on the test split, so both encode identically; levels unseen in training map to $-1$. Columns whose levels are numbers stored as strings (\textsc{Boston}'s \texttt{RAD} and \texttt{CHAS} are nominal in the OpenML table, but their levels are the integers $1$--$8$, $24$ and $0$/$1$) are read as numbers, so their order and spacing survive; only levels that are genuinely non-numeric, such as \textsc{Fiat}'s trim level \texttt{pop}/\texttt{lounge}/\texttt{sport}, receive codes. Two remarks. First, integer encoding imposes an arbitrary order on levels that have none, which costs TabDPT some accuracy on the affected tables (the ones with a nonzero categorical count in Table~\ref{tab:datasets}), but the same encoding is used in every term of both consistency checks, so it cannot by itself open a gap between two factorization orders or between a mixture and its components. Second, the target $b$ enters the conditional models as an appended feature and receives no special treatment: in the regression checks it is a continuous column passed as is, and in the classification checks it is a categorical column encoded exactly like any other, with the conditioning levels drawn from the training levels.

\textbf{Target pairs.} The two targets are chosen so that they plausibly depend on each other given the remaining features, which is what makes the consistency checks non-trivial: if $a$ and $b$ were conditionally independent given $\mathbf{x}$, the conditional and the marginal head would coincide and both conditions would hold for trivial reasons. For classification, \textsc{Wine} pairs the quality grade with the wine color, \textsc{Diamonds} the cut grade with the color grade, \textsc{Car} the acceptability rating with the safety rating, \textsc{Nursery} the admission recommendation with the child's health status, \textsc{Anneal} the annealing grade of a steel coil with the steel grade it was rolled from, \textsc{Phishing} the verdict on a website with the behavior of its server form handler, \textsc{MIC} the outcome of a myocardial infarction with whether aspirin was administered in the ICU, \textsc{Customer} whether a consumer loan defaulted with the borrower's education level, \textsc{Credit} the good/bad risk verdict with the balance band of the applicant's checking account, and \textsc{Marketing} whether a customer accepted the last campaign with that customer's marital status. Where several candidate columns existed we took the one most strongly associated with the main target, measured by Cram\'er's $V$ on the full table, subject to its class distribution not being so skewed that $\hat{p}(b \mid \mathbf{x})$ collapses onto a single class, which would make the mixture in~\eqref{eq:marginalization-consistency} trivially equal to one of its components. For regression the pairs are chosen by correlation: \textsc{Concrete} pairs compressive strength with water content ($\rho=-0.29$), \textsc{Boston} the median home value with the average number of rooms ($\rho=0.70$), \textsc{Stock} two correlated company prices ($\rho=0.71$), \textsc{Kin8nm} the end-effector distance with the first joint angle ($\rho=-0.14$), \textsc{QSAR} the acute fish toxicity of a molecule with its octanol--water partition coefficient ($\rho=0.65$), \textsc{Fiat} the asking price of a used car with its odometer reading ($\rho=-0.86$), and \textsc{Insurance} the annual billed medical expenses with the patient's body-mass index ($\rho=0.20$). The continuous case carries a second requirement that has no analog in the categorical one: $b$ must not be a deterministic function of the remaining features, or the atoms discretizing $\hat{p}(b \mid \mathbf{x})$ all collapse onto one value and the mixture reduces to a single conditional. We therefore also check that $b$ is not perfectly predictable from $\mathbf{x}$, which is what excludes tables from designed experiments, where the design variables determine one another exactly.

\textbf{Curated re-uploads.} \textsc{Anneal}, \textsc{Phishing}, \textsc{MIC}, \textsc{Customer}, \textsc{Credit} and \textsc{Marketing}, together with the regression tables \textsc{QSAR}, \textsc{Fiat} and \textsc{Insurance}, are the versions of these tables curated for the TabArena tabular-ML benchmark, which is why they are addressed by numeric OpenML ID rather than by name and version: the curation fixes column types and level names that the older uploads leave ambiguous, and the IDs pin the curated snapshot. They also broaden the classification suite along the two axes on which the original four are narrow: \textsc{Anneal}, \textsc{MIC} and \textsc{Marketing} are wide ($37$, $51$ and $24$ features) and mix numerical with categorical columns, and \textsc{MIC} carries an eight-class target whose classes are heavily imbalanced ($86\%$ of its rows fall in the majority class), so the consistency gaps are measured over sharper as well as flatter predictive distributions. \textsc{Marketing} is also the one table whose second target has levels too rare to be learned: three of the eight values of \texttt{Marital\_Status} occur in fewer than four rows. We keep them rather than merge them, since the model simply assigns them negligible mass in $\hat{p}(b \mid \mathbf{x})$.

\textbf{Wine.} This is the only dataset that is not a single OpenML table: the red-wine and white-wine quality tables are concatenated and an additional binary column \texttt{color} records which table a row came from, which supplies the second target. The two tables share their eleven physicochemical features but label their quality column differently: the red table stores the raw quality scores $3$--$8$, whereas the white table stores them index-encoded as $1$--$7$, so the merged target has eight levels rather than the seven distinct quality scores of the underlying data.

\clearpage
\section{Discretization of the Regression Checks}\label{app:discretization}

Section~\ref{sec:experiments} states that the regression checks compare distributions on a shared $20$-cell grid (marginalization) and a shared $128\times128$ grid (factorization). This appendix gives the exact construction of those cells and, since the models expose their predictive distributions in three different parameterizations, how each head is reduced to the single interface the checks consume. Every step below is a deterministic function of the fitted models: there is no sampling anywhere in the pipeline, so repeating a run reproduces the reported numbers exactly.

\textbf{Common interface: the quantile curve.} The checks never touch a model's internals. Each head is reduced to a \emph{quantile curve}: its inverse CDF $\hat{F}^{-1}$ evaluated on one fixed ladder of $L=2{,}000$ equally spaced levels running from $\alpha_1 = 2.5\times10^{-4}$ to $\alpha_L = 0.99975$. Consecutive levels therefore bracket $5\times10^{-4}$ of predicted mass, and at most $2.5\times10^{-4}$ is left beyond the ladder in each tail. A head may return a slightly non-monotone level sequence (quantile crossing); we take a running maximum along the ladder, which restores a valid inverse CDF and acts as the identity whenever the head is already monotone. Everything downstream is a function of the pairs $(\alpha_\ell, q_\ell)$ alone, so all models enter through the same interface and no model-specific tolerance or tuning is involved.

\textbf{Per-model access to $\hat{F}^{-1}$.} The heads parameterize the predictive distribution differently, and each is inverted in the way its own parameterization prescribes.

\begin{itemize}
    \item \textbf{TabPFNv2 and v3} emit a full-support bar distribution: a softmax over $5{,}000$ bars, piecewise-constant within a bar, with half-normal tails attached to the outermost two. Quantiles come from the model's own inversion of that CDF. Because the density is available in closed form, it is also what the reported NLL uses directly, with no finite-difference approximation.
    \item \textbf{TabICLv2} emits a quantile distribution: $999$ predicted knots at levels $k/1{,}000$ for $k=1,\dots,999$, joined by a monotone spline in the interior and extended by generalized-Pareto tails beyond the outermost knots. Its quantile function is the head's native output, so evaluating the ladder is a direct call; its density is recovered from the reciprocal slope of the quantile curve.
    \item \textbf{TabDPT} emits a softmax over $2{,}048$ equal-width bins spanning a fixed interval of the \emph{standardized} target space. Its public predictor returns only the bin-center expectation, which is a point prediction and cannot support a distributional check; we therefore read the bin weights off the regression head itself and invert the resulting piecewise-linear CDF, mapping the bin edges back to raw target units with the same standardization the head applies to its training targets. Its density, like TabPFN's, is exact.
\end{itemize}

A head's native resolution never binds the comparison: the coarsest of the three, TabICLv2's $999$ knots, is already some fifty times finer than the $20$ cells on which $\mathrm{TV}$ is finally evaluated.

\textbf{The $20$ cells (marginalization).} The grid is constructed per test instance \emph{and per model}, from that model's own predictions on that instance, rather than from any fixed global range. Fix an instance $\mathbf{x}_i$, write $q^{\mathrm{dir}}$ for the quantile curve of the direct marginal $\hat{p}(a\mid\mathbf{x}_i;\ctx^{A}_{-B})$, and write $q^{(k)}$ for the curve of the conditional $\hat{p}(a\mid b_k,\mathbf{x}_i;\ctx^{A})$ at the $k$-th of the $K=1{,}000$ equal-mass atoms. The grid spans the union of all $K+1$ supports,
\[
    \begin{aligned}
        \mathrm{lo}_i &= \min\big\{q^{\mathrm{dir}}_1,\ \textstyle\min_k q^{(k)}_1\big\}, \\
        \mathrm{hi}_i &= \max\big\{q^{\mathrm{dir}}_L,\ \textstyle\max_k q^{(k)}_L\big\},
    \end{aligned}
\]
that is, the smallest $\alpha_1$-quantile and the largest $\alpha_L$-quantile among every distribution entering the comparison. It is cut into $S=20$ cells of equal width by the edges $e_g = \mathrm{lo}_i + (g-1)(\mathrm{hi}_i - \mathrm{lo}_i)/S$ for $g = 1,\dots,S+1$, with a $10^{-9}$ widening guarding the degenerate case $\mathrm{hi}_i \le \mathrm{lo}_i$ of a head that predicts a single point.

Each CDF is evaluated at these edges by linear interpolation of its own $(q_\ell, \alpha_\ell)$ pairs, clamped to $0$ below $q_1$ and to $1$ above $q_L$. Writing $F(e_g)$ for those edge values, the cell masses are
\[
    \begin{aligned}
        m_1 &= F(e_2), \\
        m_g &= F(e_{g+1}) - F(e_g), \qquad 1 < g < S, \\
        m_S &= 1 - F(e_S).
    \end{aligned}
\]
The boundary cells thus absorb the mass lying outside the grid, including the $\le2.5\times10^{-4}$ per side left beyond the ladder, so both discretized distributions sum to exactly $1$ and no mass is discarded. The reported gap is $\mathrm{TV} = \tfrac12\sum_{g=1}^{S}\lvert m^{\mathrm{mix}}_g - m^{\mathrm{dir}}_g\rvert$.

Two properties of this construction matter for reading the numbers. The grid is \emph{shared}: both sides are discretized on the same edges with the same folding rule, so the boundary treatment is common to the two sides and cannot by itself open a gap. And the grid is \emph{model-specific}: because its endpoints are read off the model's own predicted support, a cell has a different width for different models, datasets and instances. This does not impair comparability, because $\mathrm{TV}$ always compares two distributions produced by the same model on the same grid and is unit-free.

\textbf{The $128\times128$ cells (factorization).} Here the two axes are discretized separately, and each axis is spanned by its own marginal head alone: the $129$ $a$-edges are equally spaced across $[q^{a}_1, q^{a}_L]$, the $\alpha_1$- and $\alpha_L$-quantiles of $\hat{p}(a\mid\mathbf{x}_i;\ctx^{A}_{-B})$, and the $b$-edges are built the same way from $\hat{p}(b\mid\mathbf{x}_i;\ctx^{B}_{-A})$. The conditionals do not widen the grid, since each is evaluated only to fill one line of the joint table.

Each axis then yields both cell masses, folded at the boundaries exactly as above, and one representative value per cell. The representative is placed by midpoint quadrature \emph{in probability space} rather than at the geometric cell center: if a cell's edges carry cumulative levels $F(e_g)$ and $F(e_{g+1})$, its representative is $\hat{F}^{-1}\big(\tfrac12(F(e_g) + F(e_{g+1}))\big)$, read off the same quantile curve. A cell straddling a peak of the density is thus represented by a point near that peak instead of by its geometric middle.

Order~$b \to a$ builds the joint one line at a time: for each $b$-cell, with mass $w_l$ and representative $b^{*}_l$, the conditional $\hat{p}(a\mid b^{*}_l, \mathbf{x}_i;\ctx^{A})$ is evaluated, converted to $a$-axis cell masses on the $a$-edges by the same folding rule, and scaled by $w_l$; the assembled table sums to $1$. Order~$a \to b$ is built symmetrically over the $a$-axis and transposed onto the same grid, so that the two tables are indexed by identical $(a, b)$ cells. The reported gap is the half-sum of absolute differences over all $128^2$ of them.

\clearpage
\section{Quantifying the Discretization Error}\label{app:discretization-error}

The quantile ladder, the atoms and the grid each approximate, so even a perfectly consistent model could score $\mathrm{TV}>0$ under the construction of Appendix~\ref{app:discretization}. We measure this floor by running the unmodified checks on a model that is consistent by construction: Gaussian process regression~\citep{Rasmussen2006.Gaussian} with two outputs $(a, b)$ over inputs $\mathbf{x}\in[0,1]^2$, coupled by an intrinsic coregionalization kernel~\citep{Alvarez2012.Kernels},
\[
    \Cov\big[y_s(\mathbf{x}), y_t(\mathbf{x}')\big] = C_{st}\,\big(k(\mathbf{x},\mathbf{x}') + \sigma^2\,\delta_{\mathbf{x}\mathbf{x}'}\big), \qquad C = \begin{pmatrix} 1 & \rho \\ \rho & 1 \end{pmatrix},
\]
where $k$ is a squared-exponential kernel with length scale $0.3$, $\sigma^2 = 0.1$ is the observation noise and $\delta$ the Kronecker delta. The hyperparameters are fixed, so nothing is fitted: the GP is conditioned on $100$ context rows drawn from its own prior and queried at $20$ test inputs. Every head the checks request is a GP prediction from this context. The marginals $\hat{p}(a\mid\mathbf{x})$ and $\hat{p}(b\mid\mathbf{x})$ are the predictive distributions at $\mathbf{x}$, and the conditional $\hat{p}(a\mid b,\mathbf{x})$ is the predictive distribution of $a$ once the observation $b$ at $\mathbf{x}$ has been added to the context (and symmetrically for $\hat{p}(b\mid a,\mathbf{x})$). All four heads are conditionals of one Gaussian process, so both consistency conditions hold exactly and any nonzero $\mathrm{TV}$ is discretization error. The GP also satisfies~\eqref{eq:column-removal}: the process $z = (b - \rho a)/\sqrt{1-\rho^2}$ is independent of $a$ under this covariance, so the training values of $b$ carry no information about $a$ beyond those of $a$ itself. The marginal heads are therefore the same whether they are computed from the full context or, as for a TFM, from the context without the other target, and both choices give identical $\mathrm{TV}$. Shifting or rescaling either target moves the grid, the atoms and the cell representatives along with it, so the $\mathrm{TV}$ of a Gaussian predictive depends on its correlation alone, which for this kernel is $\rho$ at every input. Every test instance therefore scores the same value, and the floor we report holds for any consistent model whose predictive distributions are Gaussian with that correlation.

\Cref{fig:gp-floor} shows the result. At $\rho=0$ the conditionals coincide with the marginals and both checks return $\mathrm{TV}<10^{-12}$, which is zero up to floating-point error. For $\rho\in\{0.5, 0.9, 0.99\}$ at the resolutions behind \Cref{tab:consistency}, the floor is at most $4.5\times10^{-4}$ for marginalization and $1.5\times10^{-3}$ for factorization. That is more than $35$ and $90$ times smaller, respectively, than the smallest mean gap of any TFM ($0.0163$ and $0.1307$), so the TFM gaps are not artifacts of the discretization. Refining every resolution jointly by a factor $c$ ($K=1{,}000c$ atoms, $128c\times128c$ factorization cells and a ladder of $L=2{,}000c$ levels, with the $20$ marginalization cells held fixed) drives the floor towards zero: the marginalization floor halves with every doubling of $c$ and reaches $5.3\times10^{-5}$ at $c=8$, while the factorization floor falls to $2.7\times10^{-4}$. A consistent model thus attains $\mathrm{TV}=0$ in the limit, and at the paper's resolution it comes within $1.5\times10^{-3}$ of it. The floor has two sources. The ladder spreads the mass of each quantile bin uniformly across the bin, which misplaces mass in the tails, where the bins are wide. The atoms (marginalization) or cell representatives (factorization) approximate the integral over $b$, and this error grows with $\rho$, because a sharper conditional changes faster across a cell. It dominates on the coarse factorization grids ($c<1$), where the floor accordingly rises with $\rho$.

\begin{figure}[htp]
    \centering
    \includegraphics[width=\textwidth]{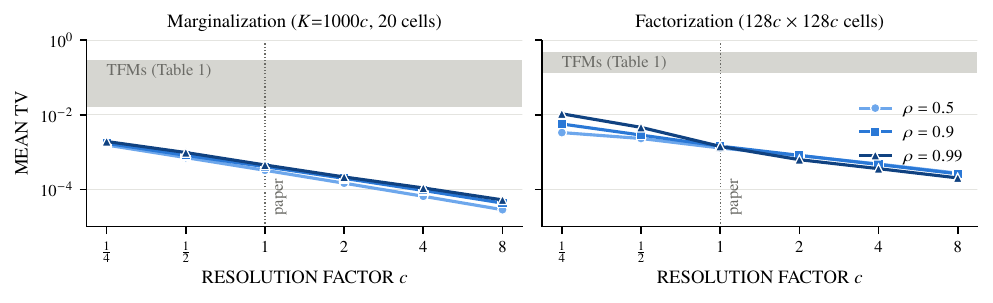}
    \caption{Discretization floor of the regression checks on a two-output GP that is consistent by construction, so every plotted value is numerical error. One line per correlation $\rho$; $\rho=0$ gives $\mathrm{TV}<10^{-12}$ and is omitted from the log axis. The resolution factor $c$ refines all discretization parameters jointly; $c=1$ (dotted) is the setting of \Cref{tab:consistency}. The gray band spans the TFMs' mean regression gaps in \Cref{tab:consistency}.}\label{fig:gp-floor}
\end{figure}

\clearpage
\section{Predictive Quality versus Consistency}\label{app:acc-tv}
\Cref{fig:acc-tv-factorization,fig:acc-tv-marginalization} plot each model's classification accuracy against its factorization and marginalization TV gap, one panel per dataset; \Cref{fig:rmse-tv-factorization,fig:rmse-tv-marginalization} do the same for regression, plotting RMSE in place of accuracy. Across all four figures, more accurate (or lower-RMSE) models are not systematically more consistent within a dataset: the two quantities do not trade off, but neither does either one predict the other from a single panel.

The dataset-level view above averages over instances, which leaves open how the inconsistency is distributed over rows. \Cref{fig:nll-tv} resolves this at the level of single test rows. For each instance we pair the negative log-likelihood that the direct predictive $\hat{p}(a \mid \mathbf{x}; \ctx^{A}_{-B})$ assigns to the true label with the TV distance between that predictive and the marginalized mixture on the right-hand side of~\eqref{eq:marginalization-consistency}.

Instances whose NLL is near zero are those on which the direct head is already close to a point mass on one class. Mixing over $b$ cannot move a saturated predictive far, so the TV on such rows is bounded near zero for arithmetic reasons alone, independently of whether the model is coherent. Keeping them would draw a diagonal band into the bottom-left corner of every panel and manufacture the appearance of TV growing steadily with the loss, an artifact of saturation rather than evidence that hard instances are the inconsistent ones. We therefore drop every instance with an NLL below $0.05$, i.e., every instance on which the model already assigns roughly $95\%$ or more to the correct class, and plot only the remainder. 

\Cref{fig:nll-tv} shows that there is no correlation between the likelihood assigned to the correct label and marginalization consistency at the per-row level.

\begin{figure}[htp]
    \centering
    \includegraphics[width=0.6\textwidth]{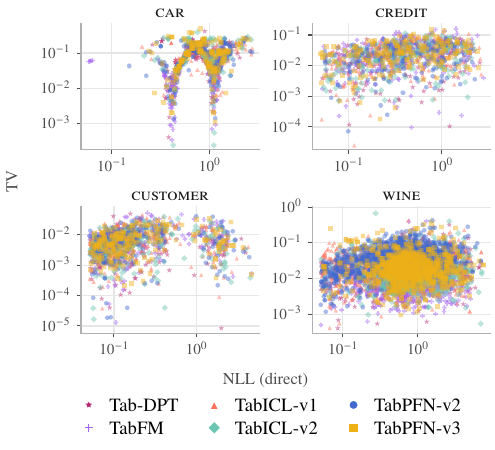}
    \caption{Per-instance predictive loss versus marginalization-consistency gap. The horizontal axis is the NLL that the direct predictive assigns to the true label, the vertical axis the TV distance between that predictive and the mixture obtained by marginalizing over the intermediate target $b$; both axes are logarithmic. Instances with an NLL below $0.05$ (about $95\%$ probability on the correct class) are omitted.}\label{fig:nll-tv}
\end{figure}

\begin{figure*}[t]
    \centering
    \includegraphics[width=\textwidth]{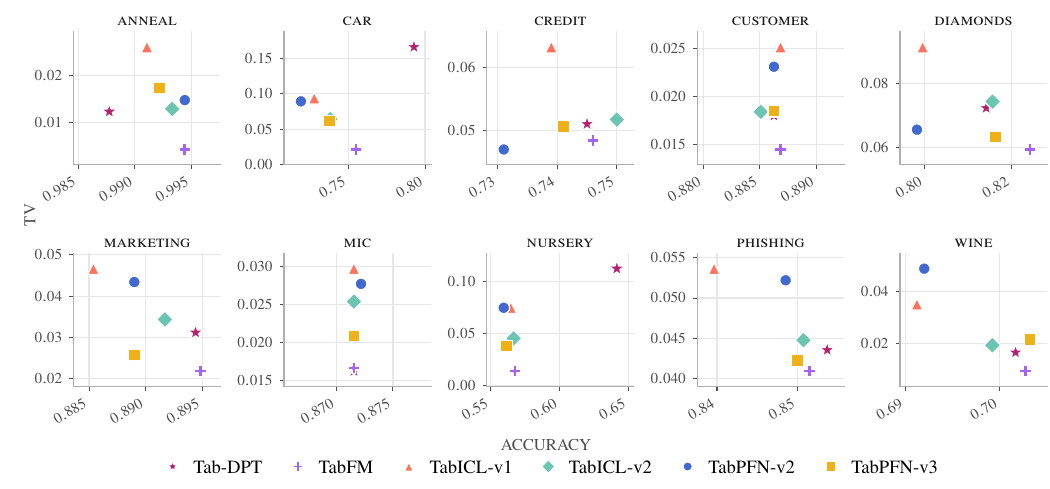}
    \caption{Predictive accuracy versus factorization-consistency TV (\Cref{tab:consistency}, bottom), one panel per classification dataset and one marker per model.}\label{fig:acc-tv-factorization}
\end{figure*}

\begin{figure*}[t]
    \centering
    \includegraphics[width=\textwidth]{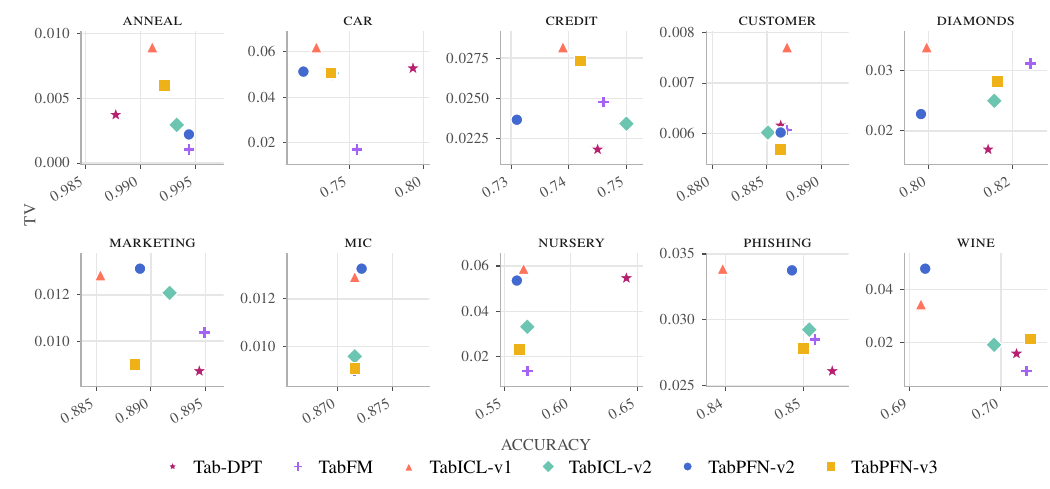}
    \caption{Predictive accuracy versus marginalization-consistency TV (\Cref{tab:consistency}, top), one panel per classification dataset and one marker per model.}\label{fig:acc-tv-marginalization}
\end{figure*}

\begin{figure*}[t]
    \centering
    \includegraphics[width=0.8\textwidth]{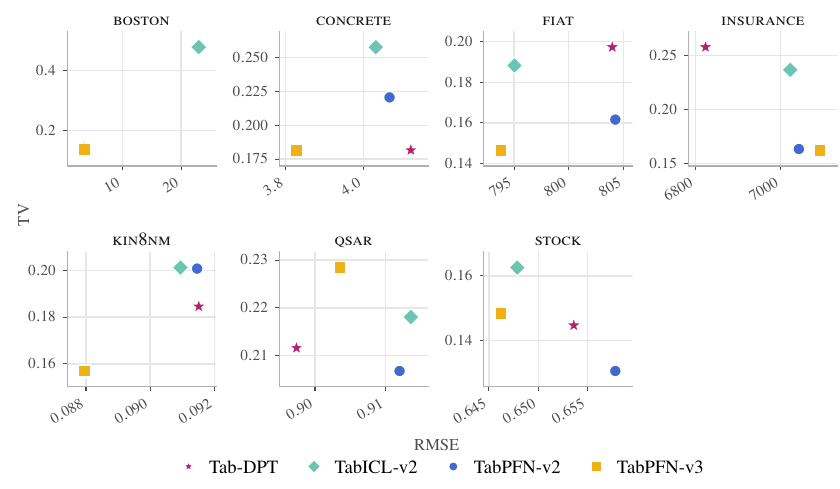}
    \caption{Predictive RMSE versus factorization-consistency TV (\Cref{tab:consistency}, bottom), one panel per regression dataset and one marker per model.}\label{fig:rmse-tv-factorization}
\end{figure*}

\begin{figure*}[t]
    \centering
    \includegraphics[width=0.8\textwidth]{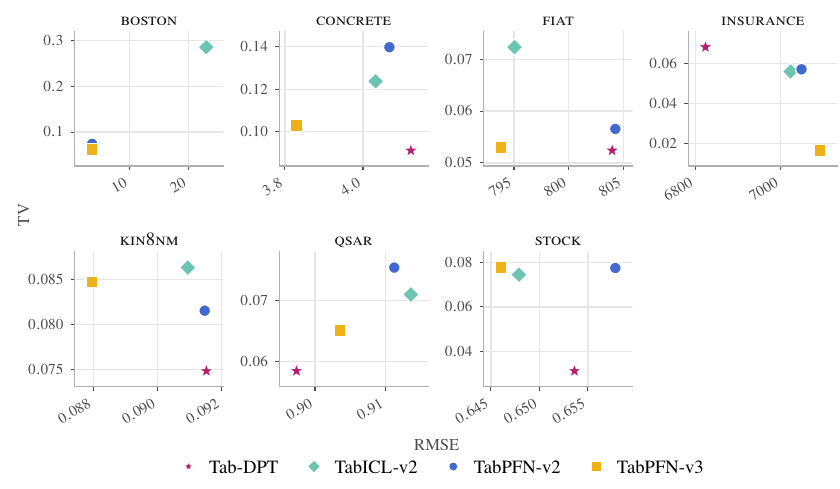}
    \caption{Predictive RMSE versus marginalization-consistency TV (\Cref{tab:consistency}, top), one panel per regression dataset and one marker per model.}\label{fig:rmse-tv-marginalization}
\end{figure*}

\clearpage
\section{Consistency versus Number of Features}\label{app:tv-features}
\Cref{fig:tv-features-factorization-classification,fig:tv-features-marginalization-classification,fig:tv-features-factorization-regression,fig:tv-features-marginalization-regression} plot each run's consistency gap against the number of features $D$ of its dataset, and show no evidence that wider tables are harder to stay consistent on: the pooled Spearman correlation between $D$ and TV is negative in three of the four settings (classification $\rho=-0.59$ for factorization and $\rho=-0.75$ for marginalization, regression $\rho=-0.50$ for factorization) and positive only for marginalization on regression ($\rho=+0.40$).

\begin{figure}[htp]
    \centering
    \begin{minipage}[t]{0.48\textwidth}
        \centering
        \includegraphics[width=\linewidth]{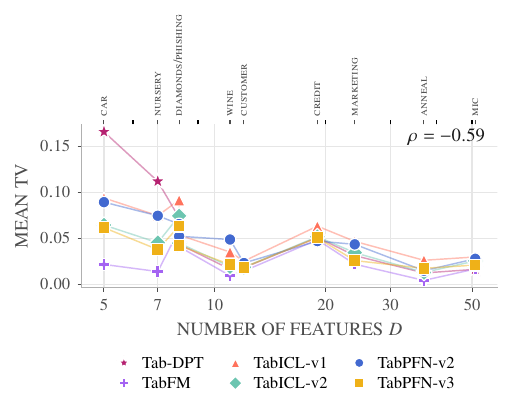}
        \caption{Factorization-consistency TV versus the number of features $D$ on the classification datasets, one marker per (dataset, model) run.}\label{fig:tv-features-factorization-classification}
    \end{minipage}\hfill
    \begin{minipage}[t]{0.48\textwidth}
        \centering
        \includegraphics[width=\linewidth]{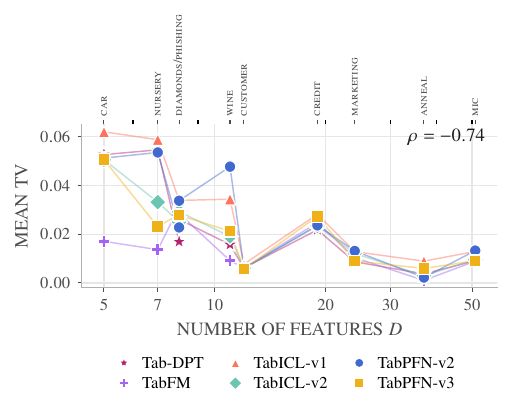}
        \caption{Marginalization-consistency TV versus the number of features $D$ on the classification datasets, one marker per (dataset, model) run.}\label{fig:tv-features-marginalization-classification}
    \end{minipage}
\end{figure}

\begin{figure}[htp]
    \centering
    \begin{minipage}[t]{0.48\textwidth}
        \centering
        \includegraphics[width=\linewidth]{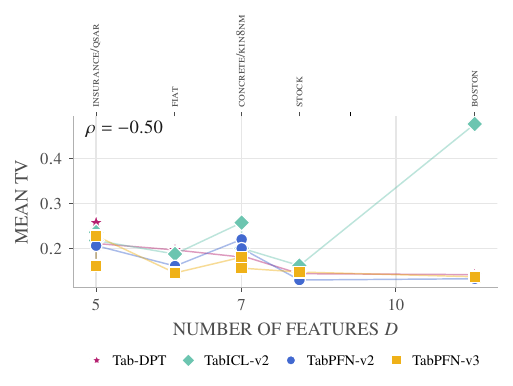}
        \caption{Factorization-consistency TV versus the number of features $D$ on the regression datasets, one marker per (dataset, model) run.}\label{fig:tv-features-factorization-regression}
    \end{minipage}\hfill
    \begin{minipage}[t]{0.48\textwidth}
        \centering
        \includegraphics[width=\linewidth]{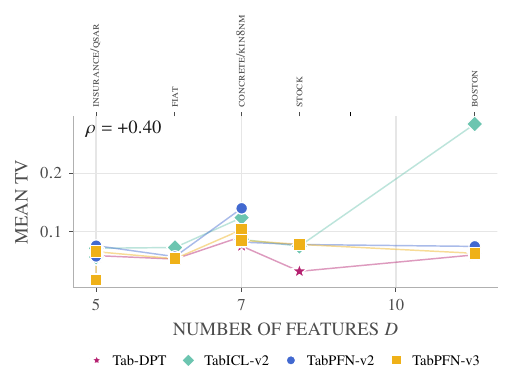}
        \caption{Marginalization-consistency TV versus the number of features $D$ on the regression datasets, one marker per (dataset, model) run.}\label{fig:tv-features-marginalization-regression}
    \end{minipage}
\end{figure}

\end{document}